\documentclass[conference]{IEEEtran}
\ifCLASSINFOpdf
\else
\fi

\usepackage{times}

\usepackage{graphicx, subfigure} 
\usepackage{amsmath}
\usepackage{amsthm}
\usepackage{amssymb}
\usepackage{fancyhdr}
\usepackage{mathrsfs}
\usepackage{epsfig}
\usepackage{graphics}
\usepackage{epsf}
\usepackage{multirow} 
\usepackage{mathtools}
\usepackage{algorithm,algorithmic}
\usepackage{array}

\newcommand{\PreserveBackslash}[1]{\let\temp=\\#1\let\\=\temp}
\newcolumntype{C}[1]{>{\PreserveBackslash\centering}p{#1}}
\newcolumntype{R}[1]{>{\PreserveBackslash\raggedleft}p{#1}}
\newcolumntype{L}[1]{>{\PreserveBackslash\raggedright}p{#1}}
\allowdisplaybreaks

\newtheorem{lemma}{Lemma}
\newtheorem{thm}{Theorem}

\def \a {\mathbf{a}}

\def \u {\mathbf{u}}

\def \v {\mathbf{v}}

\def \w  {\mathbf{w}}

\def \x {\mathbf{x}}

\def \bqsa {\begin{eqnarray}}
\def \eqsa {\end{eqnarray}}
\def \bqs {\begin{equation}\begin{aligned}}
\def \eqs {\end{aligned}\end{equation}}

\begin{document}
\title{Robust Cost-Sensitive Learning for Recommendation with Implicit Feedback}


\author{\IEEEauthorblockN{Peng Yang\IEEEauthorrefmark{1},
Peilin Zhao\IEEEauthorrefmark{1},
Xin Gao\IEEEauthorrefmark{1} and Yong Liu\IEEEauthorrefmark{1}}
}

\maketitle

\begin{abstract}

Recommendation is the task of improving customer experience through personalized recommendation based on users' past feedback.
In this paper, we investigate the most common scenario: the user-item (U-I) matrix of implicit feedback (e.g. clicks, views, purchases). Even though many recommendation approaches are designed based on implicit feedback, they attempt to project the U-I matrix into a low-rank latent space, which is a strict restriction that rarely holds in practice. In addition, although misclassification costs from imbalanced classes are significantly different, few methods take the cost of classification error into account. To address aforementioned issues, we propose a robust framework by decomposing the U-I matrix into two components: (1) a low-rank matrix that captures the common preference among multiple similar users, and (2) a sparse matrix that detects the user-specific preference of individuals. To minimize the asymmetric cost of error from different classes, a cost-sensitive learning model is embedded into the framework. Specifically, this model exploits different costs in the loss function for the observed and unobserved instances. We show that the resulting non-smooth convex objective can be optimized efficiently by an accelerated projected gradient method with closed-form solutions. Morever, the proposed algorithm can be scaled up to large-sized datasets after a relaxation. The theoretical result shows that even with a small fraction of 1's in the U-I matrix $M\in\mathbb{R}^{n\times m}$, the cost-sensitive error of the proposed model is upper bounded by $O(\frac{\alpha}{\sqrt{mn}})$, where $\alpha$ is a bias over imbalanced classes. Finally, empirical experiments are extensively carried out to evaluate the effectiveness of our proposed algorithm. Encouraging experimental results show that our algorithm outperforms several state-of-the-art algorithms on benchmark recommendation datasets.

\end{abstract}

\vspace{-0.001in}
\section{Introduction}
\vspace{-0.001in}

Recommender systems are attractive for content providers, through which they can increase sales or views. For example, online shopping websites like Amazon and eBay provide each customer personalized recommendation of products;
video portals like YouTube recommend latest popular movies to the audience. The user-item (U-I) matrix, the dominant framework for recommender systems, encodes the individual preference of users for items in a matrix. Recent work focuses on scenarios where users provide explicit feedback, e.g. ratings in different scores (such as a 1-5 scale in Netflix). However, in many real-world scenarios, most often the feedback is not explicitly but implicitly expressed with binary (0 or 1) values, indicating whether a user clicked, viewed or purchased an item. Such implicit feedback is already available in a number of information systems, such as web page logs.

In the U-I matrix of implicit feedback, the observed feedbacks corresponding to 1's represent positive instances, while the other entries annotated as 0's are unlabeled ones (a mixture of negative and missing values). Existing methods solve this problem with a binary classification, among which the most widely used technique is collaborative filtering (CF)~\cite{hu2008collaborative}, including Bayesian belief nets CF models~\cite{miyahara2000collaborative}, clustering CF models~\cite{ungar1998clustering}, regularized matrix factorization~\cite{hu2008collaborative}, etc. However, all these algorithms have two major limitations.
First, very few positives are observed, which limits the application of some classical supervised classification approaches. Besides, it is often a highly class-imbalanced learning problem as the number of observed positives is significantly smaller than that of hidden negatives, which brings a challenge to existing schemas using regular classification techniques. Second, traditional methods basically constrain the U-I matrix into a low-rank structure, which may be too restrictive and rarely holds in practice, as users with different specific preference often exist. Outliers or noisy patterns in U-I matrix may corrupt the low-rank latent space~\cite{hsu2011robust}, thus reasonable basis cannot be learned in general~\cite{linden2003amazon}.

To overcome these two limitations, in recent years, algorithms have been proposed to tackle each limitation separately.
For the first limitation (\emph{imbalanced classes}), cost-sensitive learning~\cite{elkan2001foundations} and positive-unlabeled matrix completion algorithms~\cite{hsieh2015pu} have been proposed to exploit an asymmetric cost of error from positive and unlabeled samples. For the second limitation (\emph{outlier estimation}), robust matrix decomposition algorithm~\cite{hsu2011robust} has been developed to detect the outliers and learn basis from the recovered subspace, and it showed success in different applications, including system identification~\cite{chandrasekaran2011rank}, multi-task learning~\cite{yang2017robust,yang2016learning}, PCA~\cite{candes2011robust} and graphical modeling~\cite{chandrasekaran2010latent}. Although each aforementioned algorithm, e.g. cost-sensitive learning or robust matrix decomposition, overcomes one problem, it will suffer from the other one. Thus, there is a need to devise a new and unified algorithm that can overcome both limitations simultaneously.

In this work, we intend to simultaneously perform outlier estimation and cost-sensitive learning to uncover the missing positive entries. The basis of our idea is that the U-I matrix $M$ can be decomposed into two components: the first one recognizes user-specific preference $V$, while the common preference can be better captured by the second component $M-V$. An entry-wised 1-norm regularization is applied on the first component to track the outliers of personalized traits, while a trace-norm regularization is imposed on the second one to capture a low-rank latent structure. To minimize the asymmetric cost of error from different classes, a cost-sensitive learning model is embedded into the recommendation framework. In particular, the model imposes different costs in the loss function for the observed and unobserved samples~\cite{natarajan2013learning}. The resulting non-smooth convex objective can be optimized efficiently by an accelerated projected gradient method, which can derive closed-form solutions for both components. In addition, the proposed framework can be scaled up to large-sized datasets after a relaxation. We provide a theoretical evaluation based on cost-sensitive measures by giving a proof that our algorithm yields strong cost bounds: given a U-I matrix $M\in\mathbb{R}^{n\times m}$, the expected cost of the proposed model is upper bounded by $O(\alpha/\sqrt{mn})$, where $\alpha$ is a bias over imbalanced classes. This implies that even if a small fraction of 1's are observed, we can still estimate the U-I matrix $M$ accurately when $\sqrt{mn}$ is large enough. Extensive experiments are conducted to evaluate its performance on the real-world applications. The promising empirical results show that the proposed algorithm outperforms the existing state-of-the-art baselines on a number of recommender system datasets.

\vspace{-0.001in}
\section{Related Work}
\vspace{-0.001in}

The proposed work in this paper is mainly related to two directions of research in the data mining field: (i) the approaches used in recommender systems, (ii) cost-sensitive classification over class-imbalanced matrices.

\subsection{Recommender Systems}

Generally speaking, recommender systems can be categorized by two different strategies:
\emph{content-based approach} and \emph{collaborative filtering}:
(1) Existing content-based approaches make predictions by analyzing the content of textual information, such as demography~\cite{krulwich1997lifestyle}, and finding information in the content~\cite{melville2002content}.
However, content-based recommender systems rely on external information that is usually unavailable, and they generally increase the labeling cost and implementation complexity~\cite{popescul2001probabilistic}.
(2) Collaborative filtering (CF), which we focus on in this paper, relies only on past user behavior without requiring the creation of explicit profiles. There are two main categories of CF techniques: memory-based CF and model-based CF. Memory-based CF uses the U-I rating matrix to calculate the similarity between users or items and recommend those with top-ranked similarity values, e.g. neighbor-based CF, item-based or user-based top-N recommendations. However, the prediction of memory-based CF is unreliable when data are sparse and the common items are very few.
To overcome this issue, model-based CF has been extensively investigated. Model-based CF techniques exploit the pure rating data to learn a predictive model~\cite{breese1998empirical}. Well-known model-based CF techniques include Bayesian belief nets (BNs)~\cite{miyahara2000collaborative}, clustering CF models~\cite{ungar1998clustering}, etc.
Among the various CF techniques, matrix factorization (MF)~\cite{koren2009matrix} is the most popular one. It projects the U-I matrix into a shared latent space, and represents each user or item with a vector of latent features.
Recently, CF models have been adapted into implicit feedback settings: \cite{hu2008collaborative} exploited a regularized MF technique (WRMF) on the weighted 0-1 matrix; \cite{rendle2009bpr} proposed a personalized ranking algorithm (BPRMF) that integrated Bayesian formulation with MF and optimized the model based on stochastic gradient descent.
However, the aforementioned methods modeled user preference by a small number of latent factors, which rarely holds in real-world U-I matrices, as users with specific preference often exist.

In practical scenarios, without an effective strategy to reduce the negative effect from outliers and noise, the low-rank latent space is likely corrupted and prevented from reliable solutions. Hence, besides properly modeling the low-rank structure, how to handle the gross outliers is crucial to the performance. To solve this problem, we propose a robust framework that simultaneously pursues a low-rank structure and outlier estimation via convex optimization. Although this idea has been used in some problems~\cite{candes2011robust,hsu2011robust,chandrasekaran2010latent},  it has not been applied to recommender systems.

Our work is closely related to robust PCA and matrix completion problems, both of which try to recover a low-rank matrix from imperfect observations. A major difference between the two problems is that in robust PCA the outliers or corrupted observations commonly exist but assumed to be sparse~\cite{wright2009robust}, whereas in matrix completion the support of missing entries is given~\cite{hsieh2015pu}. In contrast to~\cite{wright2009robust,hsieh2015pu}, their problem scenarios will be both considered in our algorithm for recommendation.

\subsection{Cost-Sensitive Learning}

Cost-sensitive classification has been extensively studied in many practical scenarios~\cite{liu2006influence}, such as medical diagnosis and fraud detection. For example, rejecting a valid credit card transaction may cause some inconvenience while approving a large fraudulent transaction may have very severe consequences. Similarly, recommending an uninteresting item may be tolerated by users while missing an interesting target to users would lead to business loss. Cost-sensitive learning tends to be optimized under such asymmetric costs.
Various cost-sensitive metrics are proposed recently, among which the weighted misclassification cost~\cite{elkan2001foundations}, and the weighted sum of recall and specificity~\cite{han2011data} are considered to measure the classification performance. Even though a number of cost-sensitive learning algorithms were proposed in literature~\cite{tan1993cost,lozano2008multi}, few of them were used in collaborative filtering techniques, among which the most relevant technique is the weighted MF model, that assigns an apriori weight to some selected samples in the matrix. However, such weights are used on unobserved samples to either reduce the impact from hidden negative instances~\cite{pan2008one} or assign a confidence to potential positive ones~\cite{hu2008collaborative,rendle2009bpr,sindhwani2010one}, which are basically cost insensitive and lack theoretical guarantee.


Despite extensive works in both fields, to the best of our knowledge, this is a new algorithm that combines cost-sensitive learning and robust matrix decomposition for recommender systems. In specific, our algorithm is able to simultaneously estimate outliers and uncover missing positives from the highly class-imbalanced matrices.

\vspace{-0.001in}
\section{Problem Setting}
\vspace{-0.001in}

We assume that a user-item rating matrix $M\in\mathbb{R}^{n\times m}$, $M_{ij}\in[0,1]$ for all $(i,j)$, has a bounded nuclear norm, i.e. $\|M\|_* \leq \epsilon$ ($\epsilon > 0$), where $|n|$ and $|m|$ are numbers of items and users in rating matrix, and $\epsilon$ is a constant independent of $n$ and $m$. To simulate the \emph{binary feedback scenario}, we suppose that a 0-1 matrix $Y_{ij} = I_{(M_{ij} > q)}$, where $I_{\pi}$ is the indicator function that outputs $1$ if $\pi$ holds and $0$ otherwise, and $q\in[0,1]$ is a threshold.
We assume that only a subset of 1's of $Y$ are observed. Let $S$ be the total number of 1's in $Y$, and the observations $\Omega$ be a subset of positive entries sampled from $\{(i,j)|Y_{ij}=1\}$, then the sampling rate is $\rho = |\Omega|/S$. In particular, we use $A$ to represent the observations, where $A_{ij} = 1$ if $(i,j)\in\Omega$, and $A_{ij}=0$ otherwise. The recommendation problem is straight-forward: given the observation U-I matrix $A$, the objective is to recover $Y$ based on this observed sampling.

Inspired by the Regularized Loss Minimization (RLM) in which one minimizes an empirical loss plus a regularization term jointly~\cite{shalev2011stochastic}, we aim to find a matrix $X$ under some constraints to minimize the element-wise loss between $X$ and $A$ for each element, and it can be formulated in the following optimization problem:
\bqs\label{objCM}
\min_X \sum_{i=1}^n\sum_{j=1}^m \ell(X_{ij}, A_{ij}) + \lambda g(X),
\eqs
where $\lambda \geq 0$ is a trade-off parameter, $\ell(X,A)$ is a loss function between the matrices $X$ and $A$,
and $g(X)$ is a convex regularization term that constrains $X$ into simple sets, e.g. hyperplanes, balls, bound constraints, etc.

\section{Algorithm}
We propose to solve (\ref{objCM}) by two steps: 1) to derive a cost-sensitive loss function for $\ell(X,A)$ and
2) to learn a robust framework via a regularization term $g(X)$.

\vspace{-0.001in}
\subsection{Cost-Sensitive Learning}
\vspace{-0.001in}

Now we are ready to derive a cost-sensitive learning method for recommendation. We first introduce a new measurement. Then we exploit a learning schema by optimizing the cost-sensitive measure.

\subsubsection{Measurement}

To solve the problem (\ref{objCM}), traditional techniques~\cite{schmidt2005compound,weimer2008improving} exploited a cost-insensitive loss to minimize the mistake rates. However, in the recommendation scenario, the cost of missing a positive target is much higher than having a false-positive. Thus, we study new recommendation algorithms, which can optimize a more appropriate performance metric, such as the \emph{sum} of weighted \emph{recall} and \emph{specificity},
\bqs\label{sum_cost_sensitive}
sum = \mu_p \times recall + \mu_n \times specificity,
\eqs
where $0\leq\mu_p,\mu_n\leq 1$ and $\mu_p+\mu_n =1$. In general, the higher the \emph{sum} value, the better the performance.
Besides, another suitable metric is the total cost of the algorithm~\cite{elkan2001foundations}, which is defined as:
\bqs\label{cost_sensitive_measure}
cost = c_p \times M_p + c_n \times M_n,
\eqs
where $M_p$ and $M_n$ are the number of false negatives and false positives respectively, and $0\leq c_p,c_n\leq 1$ are the cost parameters for positive and negative classes with $c_p + c_n = 1$. The lower the cost value, the better the classification performance.

\subsubsection{Learning Method}

We propose a cost-sensitive learning function by optimizing two cost-sensitive measures above.
Before presenting the method, we first prove the following lemma that motivates our solution.

\begin{lemma}\label{cost_sensitive_lemma}
Consider a cost-sensitive learning problem in binary feedback scenario, the goal of maximizing the weighted sum in (\ref{sum_cost_sensitive}) or minimizing the weighted cost in (\ref{cost_sensitive_measure}) is equivalent to minimizing the following objective:
\bqs
\sum_{A_{ij}=+1} \alpha I_{(X_{ij}\leq q)} + \sum_{A_{ij}=0} I_{(X_{ij} > q)},
\eqs
where $\alpha=\frac{\mu_pT_n}{\mu_nT_p}$ for the maximization of the weighted sum, $T_p$ and $T_n$ are the number of positive examples and negative examples, respectively, $\alpha = \frac{c_p}{c_n}$ for the minimization of the weighted misclassification cost.
\end{lemma}
\begin{proof}
By analyzing the function of the weighted sum in (\ref{sum_cost_sensitive}), we can derive the following
\bqs\notag
sum & = \mu_p \frac{T_p - M_p}{T_p} + \mu_n \frac{T_n - M_n}{T_n} \\
    & = 1 - \frac{\mu_n}{T_n}\left[ \frac{\mu_pT_n}{\mu_nT_p}\sum_{A_{ij}=+1} I_{(X_{ij}\leq q)} + \sum_{A_{ij}=0} I_{(X_{ij} > q)} \right].
\eqs
Thus, maximizing \emph{sum} is equivalent to minimizing
\bqs\notag
\frac{\mu_pT_n}{\mu_nT_p}\sum_{A_{ij}=+1} I_{(X_{ij}\leq q)} + \sum_{A_{ij}=0} I_{(X_{ij} > q)}.
\eqs
Secondly, by analyzing the function of the weighted cost in (\ref{cost_sensitive_measure}), we can derive the following:
\bqs\notag
cost & = c_p \times M_p + c_n \times M_n \\
     & = c_n \left[ \frac{c_p}{c_n}\sum_{A_{ij}=+1} I_{(X_{ij}\leq q)} + \sum_{A_{ij}=0} I_{(X_{ij} > q)} \right].
\eqs
\noindent
Thus, minimizing \emph{cost} is equivalent to minimizing
\bqs\notag
\sum_{A_{ij}=+1} \frac{c_p}{c_n} I_{(X_{ij}\leq q)} + \sum_{A_{ij}=0} I_{(X_{ij} > q)}.
\eqs
Thus, the lemma holds by setting $\alpha = \frac{\mu_pT_n}{\mu_nT_p}$ for sum and $\alpha = \frac{c_p}{c_n}$ for cost.
\end{proof}
Lemma \ref{cost_sensitive_lemma} gives the explicit objective for optimization, but the indicator function is not convex. To facilitate the optimization task, we replace the indicator function by its convex surrogate, i.e. the following modified loss functions:
\bqs\label{cost_loss_fun_2}
\ell^{I}(X_{ij},A_{ij}) = \alpha I_{(A_{ij}=1)}\ell(X_{ij},1) + I_{(A_{ij}=0)}\ell(X_{ij},0),
\eqs
\bqs\label{cost_loss_fun_1}
\ell^{II}(X_{ij},A_{ij}) = I_{(A_{ij}=1)}\ell(X_{ij},\alpha) + I_{(A_{ij}=0)}\ell(X_{ij},0),
\eqs
where the squared loss $\ell(x,a) = \frac{1}{2}(x - a)^2$ is the commonly used penalty, which is optimal to the Gaussian noise~\cite{zhou2010stable}.

We would observe that for $\ell^{I}(X_{ij},A_{ij})$, the slope of the loss function changes for a specific class, leading to more ``aggressive" updating; for $\ell^{II}(X_{ij},A_{ij})$, the required margin for specific class changes compared to the traditional squared loss, resulting in more ``frequent" updating.

Traditional classification approach~\cite{rendle2009bpr} treats the observed data equally with unobserved one. In U-I matrices with spare observed data, this will be vulnerable to the majority class, i.e. most of unobserved examples are in the negative class, thus the hidden positives will be more likely to be predicted as the negative class, which consequently harms the performance. Therefore, we impose different costs for false negatives and false positives in the loss function.

\vspace{-0.001in}
\subsection{Robust Low-Rank Representation with Personalized Traits}
\vspace{-0.001in}

We intend to develop a robust framework to simultaneously pursue outliers and a low-rank latent structure. To achieve this goal, we propose a robust matrix decomposition framework, assuming that the U-I matrix can be decomposed as the sum of a low-rank component $U$ and a spare component $V$, where outliers or corrupted entries can be tracked. Such additive decompositions have been used in system identification~\cite{chandrasekaran2011rank}, PCA~\cite{candes2011robust} and graphical modeling~\cite{chandrasekaran2010latent}.


Denoted by $X$ the combination of $U$ and $V$, we define a function $f(\cdot)$,
\bqs\notag 
X  =  f\left(W\right) = \begin{bmatrix} I_n, I_n \end{bmatrix}W = U + V, 0\leq U_{ij},V_{ij}\leq1, \forall(i,j),
\eqs
where $I_n\in\mathbb{R}^{n\times n}$ is an identity matrix, and $W$ is a weight matrix decomposed into two components:
\bqs\label{decompositionW}
 \{W | W = \begin{bmatrix} U \\ V \end{bmatrix}, U\in\mathbb{R}^{n\times m}, V\in\mathbb{R}^{n\times m}\},
\eqs
where $W = [\w_{1},\ldots,\w_{m}]\in\mathbb{R}^{2n\times m}$ and $\w_{i} = \begin{bmatrix} \u_{i} \\ \v_{i} \end{bmatrix}\in\mathbb{R}^{2n}$ is the $i$-th column of $W$.
Given an instance $\w_{i}$, we recover the user rating $\x_{i}$ by the summation of $\u_i$ and $\v_i$,
\bqs\notag
\x_{i} = f(\w_i) = \begin{bmatrix} I_n, I_n \end{bmatrix}\w_i = \u_i + \v_i.
\eqs
This optimization problem can be formulated as follows: given that $U$ and $V$ are unknown, but $U$ is known to be low rank and $V$ is known to be sparse, we recover $A$ with $X = U + V$,
\bqs\label{original_obj}
\min_{U,V} \hbox{rank}(U) + \lambda\|V\|_0 \quad \hbox{s.t.} \quad \mathcal{P}_\Omega(A) = \mathcal{P}_\Omega(U + V),
\eqs
where $\mathcal{P}_\Omega(\cdot)$ is projection operator on the support of binary values $\Omega\in\{0,1\}^{n\times m}$.
However, (\ref{original_obj}) is a highly nonconvex optimization problem, which is computationally intractable (NP-hard)~\cite{amaldi1998approximability}. Fortunately, we can relax (\ref{original_obj}) by
replacing the $\ell^0$-norm with the $\ell^1$-norm, and the rank with the nuclear norm $\|U\|_* = \sum_i\sigma_i(U)$,
yielding the following convex surrogate:
\bqs\label{original_obj_relax}
\min_{U,V} \|U\|_{*} + \lambda\|V\|_{1} \quad \hbox{s.t.} \quad \mathcal{P}_\Omega(A) = \mathcal{P}_\Omega(U + V).
\eqs
To minimize the asymmetric cost of errors from different classes, a cost-sensitive learning model is embedded into this framework.
This model solves a slightly relaxed version of (\ref{original_obj_relax}), in which the equality constraint is replaced with the convex cost-sensitive loss function:
\bqs\label{original_obj_final}
\min_{U,V} \quad & \lambda_1\|U\|_{*} + \lambda_2\|V\|_1 + \sum_{i,j}\ell^{*}\left(U_{ij} + V_{ij}, A_{ij}\right), \\
& \hbox{s.t.} \quad 0\leq U_{ij},V_{ij}\leq1, \quad \forall(i,j)
\eqs
where $\lambda_{1}$/$\lambda_{2}$ are non-negative trade-off parameters and the loss $\ell^*(\cdot)$ where $*\in\{I,II\}$ is given in section A.
Directly optimizing (\ref{original_obj_final}) can be computationally expensive, e.g. algorithms based on matrix computation usually scale to $O(n^3)$~\cite{meyer2000matrix}. Moreover, the objective above includes non-smooth regularization terms, $\|U\|_*$ and $\|V\|_{1}$, thus direct calculation with sub-gradient methods leads to a slow convergence rate and a lack of a practical stopping criterion \cite{bagirov2014introduction}. Alternatively, we propose an accelerated proximal gradient schema to optimize the problem and derive a closed-form solution.

\vspace{-0.001in}
\subsection{Optimization}
\vspace{-0.001in}

Although the problem above can be solved by~\cite{vandenberghe1996semidefinite}, the composite function with linear constraints has not been studied in recommender systems. Motivated by \cite{beck2009fast}, we solve the problem (\ref{original_obj_final}) with an accelerated proximal gradient line (APGL) search method, which can iteratively optimize $U$ and $V$ with closed-form solutions via minimizing the objective (\ref{original_obj_final}) of weight matrix at round $t$,
\bqs\label{objective_proximal_gradient}
W^{t+1} & = \underset{W}{\operatorname{argmin}} \sum_{i,j}\ell^{*}\left(W^t_{ij}\right) + \lambda_1\|U\|_{*} + \lambda_2\|V\|_1, \\
& \hbox{s.t.} \quad 0\leq U_{ij},V_{ij}\leq1,\forall(i,j), \quad W = \begin{bmatrix} U \\ V \end{bmatrix}
\eqs
where $\ell^{*}(W_{ij}^{t}):=\ell^*\left(U_{ij}^t+V_{ij}^t,A_{ij}\right)$.
The APGL minimizes the objective (\ref{objective_proximal_gradient}) by iteratively updating $W$ with the update rule $W^{t+1} = \Pi_{\pi}\left(W^{t} - \eta\nabla\ell^{*}_{(t)}\right)$ where $\Pi_\pi(W)$ is the projection of $W$ onto the set $\pi$ and $\eta > 0$ is a stepsize. Using the first-order Taylor expansion of $\ell^{*}(W^t)$, the problem (\ref{objective_proximal_gradient}) can be derived by the following equation given $\eta > 0$,
\bqs\label{linearization_projection}
\underset{W}{\operatorname{min}} \sum_{i=1}^m & \left( \ell^{*}(\w_i^{t}) + \left\langle \nabla_{\w_i} \ell^{*}(\w_i^t),\w_i - \w_i^{t} \right\rangle
          + \frac{1}{2\eta}\mathcal{D}(\w_i,\w_i^t)\right)  \\
          & + \lambda_1\|U\|_{*} + \lambda_2\|V\|_1,
\eqs
where $\nabla_{\w_i} \ell^{*}(\w_i^t)$ denotes the derivative of $\ell^{*}(\cdot)$ w.r.t $\w_i$ at $\w_i=\w_i^t$ and $\mathcal{D}(\cdot,\cdot)$ measures the Euclidean distance between variables.
Eq (\ref{linearization_projection}) contains non-smooth regularization terms. We next solve (\ref{linearization_projection}) with a proximal operator problem.
\begin{lemma}\label{optimation_UV}
Given $W = \begin{bmatrix} U \\ V \end{bmatrix}$, (\ref{linearization_projection}) can turn into a proximal operator problem in two steps:

\noindent
Step I:
\bqs\notag
\hat{\u}_i^{t} = \u_i^{t} - \eta\frac{\partial\ell^*(\u_i^{t}+\v_i^{t},\a_i)}{\partial\u_i^{t}},
\hat{\v}_i^{t} = \v_i^{t} - \eta\frac{\partial\ell^*(\u_i^{t}+\v_i^{t},\a_i)}{\partial\v_i^{t}},
\eqs
Step II:
\bqs\label{OptimalU}
&\tilde{U}^{t+1}  =  \underset{U}{\operatorname{argmin}} \frac{1}{2\eta_t}\sum_{i=1}^m\|\u_i - \hat{\u}_i^t\|_F^2 + \lambda_1\|U\|_{*},
\eqs
\vspace{-0.03in}
\bqs\label{OptimalV}
&\tilde{V}^{t+1}  =  \underset{V}{\operatorname{argmin}} \frac{1}{2\eta_t}\sum_{i=1}^m\|\v_i - \hat{\v}_i^t\|_F^2 + \lambda_2\sum_{i=1}^m\|\v_i\|_{1},
\eqs
where $\eta_t > 0$ is the stepsize at round $t$.
\end{lemma}
\vspace{-0.01in}
\noindent
We give the proof in the Appendix. Both Eq (\ref{OptimalU}) and (\ref{OptimalV}) admit a closed-form solution.

\subsubsection{Computation of $\tilde{U}$}

Inspired by~\cite{boyd2004convex}, we show that the solution to (\ref{OptimalU}) can be obtained via solving a simple convex optimization problem.
\vspace{-0.001in}
\begin{thm}\label{them_U}
Assume the eigendecomposition of $\hat{U}^t = P\hat{\Sigma} Q^{\top}\in\mathbb{R}^{n\times m}$ where $r = $\textrm{rank}$(\hat{U}^t)$, $P\in\mathbb{R}^{n\times r}$, $Q\in\mathbb{R}^{m\times r}$, and $\hat{\Sigma}=$\textrm{diag}$(\hat{\sigma}_1,\ldots,\hat{\sigma}_r)\in\mathbb{R}^{r\times r}$.
Let $\{\sigma_i\}_{i=1}^r, \sigma_i \geq 0$, the problem (\ref{OptimalU}) can turn into an equivalent form,
\bqs
\label{OptimalSigma}
\min_{\{\sigma_i\}_{i=1}^r} & \frac{1}{2\eta_t}\sum_{i=1}^r(\sigma_i - \hat{\sigma}_i)^2 +  \lambda_1 \sum_{i=1}^r \sigma_i.
\eqs
Its optimal solution is $\sigma_i^* = [\hat{\sigma}_i - \eta_t\lambda_1]_{+}$ for $i\in[1,r]$, where $[x]_{+}=\max(0,x)$.
Denoted by $\Sigma^* =$diag$(\sigma_1^*,\ldots,\sigma_r^*)\in\mathbb{R}^{r\times r}$, the optimal solution to Eq. (\ref{OptimalU}) is given by,
\bqs
\label{OptimalSolutionU}
\tilde{U}^* = P\Sigma^* Q^{\top}.
\eqs
\end{thm}

\subsubsection{Computation of $\tilde{V}$}

The problem in (\ref{OptimalV}) is a Lasso problem and admits a closed-form solution for each entry,
\bqs
\label{OptimalSolutionV}
\tilde{V}_{ki}^{t+1} = \mathcal{K}\left( \left[V^{t} - \eta\frac{\partial\ell^*(U^{t}+V^{t},A)}{\partial V^{t}}\right]_{ki} , \eta_t\lambda_2 \right),
\eqs
where $[\cdot]_{ki}$ returns the $(k,i)$-th element of a matrix, sgn$(\cdot)$ defines the sign function, $|\cdot|$ gives the absolute value when the argument is a scalar, and $\mathcal{K}(a,b) = [|a| - b]_{+}\hbox{sign}(a)$.
Specifically, if $\hat{V}_{ki} = |V_{ki} - \frac{\partial \ell^*}{\partial V_{ki}}| < \eta_t\lambda_2$, $\forall k\in[1,n]$, the column $\v_{i}$ decays to \textbf{0} and thus the user will obey only the low-rank structure $\u_i$; if $\hat{V}_{ki} > \eta_t\lambda_2$, non-zero entries hold in the column $\v_i$, and $\x_i$ will be the summation of $\v_i$ and $\u_i$.

Next, APGL iteratively updates $U^{t+1}$ and $V^{t+1}$ with a stepsize $\tau_{t+1} = \frac{1 + \sqrt{1 + 4\tau_{t}^2}}{2}$:
\bqs\label{update_UVE}
& U^{t+1} = \mathcal{P}_{[0,1]}\left(\tilde{U}^{t+1} + \frac{\tau_t - 1}{\tau_{t+1}}(\tilde{U}^{t+1} - \tilde{U}^{t})\right),\\
& V^{t+1} = \mathcal{P}_{[0,1]}\left(\tilde{V}^{t+1} + \frac{\tau_t - 1}{\tau_{t+1}}(\tilde{V}^{t+1} - \tilde{V}^{t})\right),
\eqs
where $\mathcal{P}_{[0,1]}(x)$ is a projection of $x$ onto the set $x\in[0,1]$. The optimization of $U$ and $V$ has a convergence rate of $O(1/t^2)$ where $t$ is the iteration number, guaranteed by the convergence property of APGL~\cite{beck2009fast}, and it is faster than subgradient methods.

\subsubsection{Algorithm}

We are ready to present a Robust Cost-Sensitive Learning for Recommendation, namely CSRR. We adopt a mistake-driven update rule: let an update be issued when an error occurs, i.e. $\ell_t^* > 0$.
When using cost-sensitive loss function (\ref{cost_loss_fun_2}), the subgradient can be expressed as:
\bqs\label{Subgradient_2}
\frac{\partial \ell^{I}}{\partial X_{ij}} = \left\{
                   \begin{array}{ll}
                     \alpha I_{(A_{ij}=1)}( X_{ij} - 1) + I_{(A_{ij}=0)}X_{ij}, \quad \ell_t^{I}>0; \\
                     0,  \quad otherwise.
                   \end{array}
                 \right.
\eqs
We refer to the algorithm above as ``CSRR-I" for short.
Specifically, when using the loss function (\ref{cost_loss_fun_1}), the subgradient can be expressed as:
\bqs\label{Subgradient_1}
\frac{\partial \ell^{II}}{\partial X_{ij}}  = \left\{
                   \begin{array}{ll}
                     I_{(A_{ij}=1)}( X_{ij} - \alpha) + I_{(A_{ij}=0)}X_{ij}, \quad \ell_t^{II}>0; \\
                     0,  \quad otherwise.
                   \end{array}
                 \right.
\eqs
We refer to the above cost-sensitive learning algorithm as ``CSRR-II" for short.
We summarize the algorithm ``CSRR-I" and ``CSRR-II" in Alg. \ref{CSRCP}.
Note that $X = f(W) = U + V$ and $\frac{\partial X}{\partial U} = \frac{\partial X}{\partial V} = 1$, we obtain $\frac{\partial \ell^{*}}{\partial U} = \frac{\partial \ell^{*}}{\partial X}\frac{\partial X}{\partial U} = \frac{\partial \ell^{*}}{\partial X}$ and
$\frac{\partial \ell^{*}}{\partial V} = \frac{\partial \ell^{*}}{\partial X}\frac{\partial X}{\partial V} = \frac{\partial \ell^{*}}{\partial X}$ $\forall*\in\{I, II\}$.

In contrast to the work~\cite{hsieh2015pu}, the proposed loss function is derived by cost-sensitive learning, and the output is the summation of the low-rank and spare structures.

\vspace{-0.001in}

\begin{algorithm}[t]
\caption{CSRR} \label{CSRCP}
\begin{algorithmic}[1]
\STATE {\bf Input}: an observed U-I matrix $A\in\mathbb{R}^{n\times m}$, the maximal number of iterations $T$, $\eta$ and $\lambda$. 
\STATE {\bf Initialize}: $U^{0}, V^{0}, \tilde{U}^{0}, \tilde{V}^{0}, \tau_{0}=1$;
\FOR{$t=0,\ldots, T$}
    \STATE {\bf Predict}: $X^t = f(W^t) = U^t + V^t$;
    \STATE {\bf Suffer Loss}: $\ell_t = \ell^*(X^t,A);$
    \IF {$\ell_t > 0$}
        \STATE  Compute subgradient $\nabla\ell^*$ with Eq (\ref{Subgradient_2}) or Eq (\ref{Subgradient_1});
        \STATE  Compute $\tilde{U}^{t+1}$ and $\tilde{V}^{t+1}$ with Eq (\ref{OptimalSolutionU}) and (\ref{OptimalSolutionV});
        \STATE  Update $U^{t+1}$, $V^{t+1}$ and $\tau_{t+1}$ with Eq (\ref{update_UVE});
    \ENDIF
\ENDFOR
\STATE {\bf Output}: $X^T$, $U^T$ and $V^T$;
\end{algorithmic}
\end{algorithm}

\subsection{An Efficient Optimization}
\vspace{-0.001in}

Though CSRR with the nuclear norm minimization (NNM) can perform stably without knowing the target rank in advance, it is limited by the necessity of executing expensive singular value decomposition (SVD) for multiple times (step 8, Alg. \ref{CSRCP}). At less expense, bilinear factorization (BF)~\cite{babacan2012sparse,wang2012probabilistic} is an alternative by replacing $U$ with $P^{\top}Q$, where the product of two factor matrices $P\in\mathbb{R}^{d\times n}$ and $Q\in\mathbb{R}^{d\times m}$ implicitly guarantees that the rank of $P^{\top}Q$ is never over $d$, typically $d\ll \min(m,n)$. Theorem \ref{NNM_to_BF} provides a bridge between NNM and BF models.
\begin{thm}\label{NNM_to_BF}
For any matrix $U\in\mathbb{R}^{n\times m}$, the following relationship holds~\cite{mazumder2010spectral}:
\bqs\notag
\|U\|_* = \min_{P,Q} \frac{1}{2}\|P\|_F^2 + \frac{1}{2}\|Q\|_F^2 \quad \hbox{s.t.} \quad U = P^{\top}Q.
\eqs
If $\hbox{rank}(U) = d \leq \min(m,n)$, then the minimum solution above is attained at a factor decomposition $U = P^{\top}Q$, where $P\in\mathbb{R}^{d\times n}$ and $Q\in\mathbb{R}^{d\times m}$.
\end{thm}

Such nuclear norm factorization is well established in recent work~\cite{recht2010guaranteed}. Substituting BF model into (\ref{original_obj_final}), we solve $P$ and $Q$ with the projected gradient scheme in Lemma \ref{optimation_UV},
\bqs\label{MF_objective}
 & \min_{P,Q} \frac{1}{2\eta} \left(\|P - \hat{P}^t\|^2_F + \|Q - \hat{Q}^t\|^2_F \right) + \frac{\lambda_1}{2}\left(\|P\|_F^2 + \|Q\|_F^2\right),\\
 & \hbox{s.t.} \quad X^t = P^{t\top}Q^t+V^t, \quad  0 \leq V_{ij},(P^{\top}Q)_{ij} \leq 1, \forall(i,j)
\eqs
where
\bqs\label{P_Q_proximal}
\hat{P}^t = P^t - \eta \frac{\partial\ell^*_t}{\partial X^t}\frac{\partial X^t}{\partial P^t}, \quad \hat{Q}^t = Q^t - \eta \frac{\partial\ell^*_t}{\partial X^t}\frac{\partial X^t}{\partial Q^t},
\eqs
for $\forall*\in\{I, II\}$. We use $L_{2}$-norm regularization to avoid overfitting when $d$ is larger than the intrinsic rank.

Eq. (\ref{MF_objective}) is biconvex, i.e. fixing $P$ the problem is convex on $Q$, and vice-versa. Their updating rule can be derived by setting the derivatives with respect to $P$ or $Q$ to zero by fixing the other one,
\bqs\label{computation_PQ}
P = \mathcal{P}_{[0,\frac{1}{\sqrt{d}}]}\left(\frac{1}{1 + \eta\lambda_1}\hat{P}^t\right), \quad 
Q = \mathcal{P}_{[0,\frac{1}{\sqrt{d}}]}\left(\frac{1}{1 + \eta\lambda_1}\hat{Q}^t\right), 
\eqs
where $\mathcal{P}_{[0,\frac{1}{\sqrt{d}}]}(\cdot)$ guarantees $(P^{\top}Q)_{ij}\in[0,1]$. The above two steps will iterate until convergence, then $V$ is updated with a projection of (\ref{OptimalSolutionV}). We refer to this refined efficient algorithm as "CSRR-e" for short. We summarize "CSRR-e" in Alg. \ref{CSRCP2}.


\subsection{Theoretical Analysis}

\begin{algorithm}[t]
\caption{CSRR-e} \label{CSRCP2}
\begin{algorithmic}[1]
\STATE {\bf Input}: an observed U-I matrix $A\in\mathbb{R}^{n\times m}$, the maximal number of iterations $T$ and $\eta$. 
\STATE {\bf Initialize}: $P^{0}, Q^{0}$ and $V^{0}$;
\FOR{$t=0,\ldots, T$}
    \STATE {\bf Predict}: $X^t = f(W^t) = P^{t\top}Q^t + V^t$;
    \STATE {\bf Suffer Loss}: $\ell_t = \ell^*(X^t,A) > 0$;
    \STATE Let $k = 0$;
    \REPEAT
        \STATE  Compute Proximal Gradient $\hat{P}^t_{(k)}$ and $\hat{Q}^t_{(k)}$ as (\ref{P_Q_proximal});
        \STATE  Update $P^{t}_{(k)}$ and $Q^{t}_{(k)}$ as Eq (\ref{computation_PQ});
        \STATE  $k = k + 1$;
    \UNTIL convergence
    \STATE  Let $P^{t+1} = P^{t}_{(k)}$, $Q^{t+1} = Q^{t}_{(k)}$;
    \STATE  Compute $V^{t+1} = \mathcal{P}_{[0,1]}(\tilde{V}^{t+1})$ with $\tilde{V}^{t+1}$ in Eq (\ref{OptimalSolutionV});
\ENDFOR
\STATE {\bf Output}: $X^T$, $V^T$, $P^T$ and $Q^T$;
\end{algorithmic}
\end{algorithm}

Although recommender systems were extensively studied recently, very few work has formally investigated it in the cost-sensitive measures. Below, we theoretically analyze the performance of the proposed CSRR in terms of the cost-sensitive loss.

To ease our discussion, we assume that $\mathcal{X} := \{ X\in\mathbb{R}^{n\times m}  | \|X\|_* \leq \epsilon, 0\leq X_{ij} \leq 1, \forall (i,j)\}$. The expected error can be formulated as $\mathbb{E}[R_{l^*}(X)] = \mathbb{E}[\frac{1}{mn}\sum_{i,j}\ell^*(X_{ij},A_{ij})]$, and the empirical error as $\hat{R}_{l^*}(X) = \frac{1}{mn}\sum_{i,j}\ell^*(X_{ij},A_{ij})$.

Inspired by the work in~\cite{hsieh2015pu}, we begin with the following theory that gives the loss bound of the proposed CSRR. The proof is given in the Appendix.
\begin{thm}\label{expectlossbound_thm}
Assume that $X\in\mathcal{X}$, then with probability at least $1 - \delta$,
\bqs\label{expectempiricalloss}
& \mathbb{E}[R_{l^*}(X)] - \min_{X\in\mathcal{X}}\hat{R}_{l^*}(X) \\
\leq & C\frac{\epsilon\alpha(\sqrt{n} + \sqrt{m} + \sqrt[4]{S})}{mn} + \frac{\alpha\sqrt{\log(2/\delta)}}{\sqrt{mn}},
\eqs
where $C$ is a constant, $\alpha$ is the bias of imbalanced classes.
\end{thm}
\noindent

\textbf{Remark:}
Given a U-I matrix $A\in\mathbb{R}^{n\times m}$, the generalization bound of CSRR is of the order to $O(\frac{\alpha}{\sqrt{mn}})$. This shows that even if a small ratio of 1's is observed in the matrix, we can achieve a lower upper-bound when $\sqrt{mn}$ becomes large enough.

Now we evaluate CSRR with the thresholded 0-1 matrix, on which the objective is defined as,
\bqs\label{cost_error}
\mathcal{L}_{\alpha,q}(X) = \alpha\sum_{A_{ij}=+1}I_{(X_{ij} \leq q)} + \sum_{A_{ij}=0}I_{(X_{ij} > q)},
\eqs
The following lemma shows that the expectation of $\mathcal{L}_{\alpha,q}(X)$ and $\mathbb{E}[R_{l^*}(X)]$ can be related by a linear transformation:
\begin{lemma}\label{cost_transformation_lamma}
Given $\min(\frac{1}{q^2}, \frac{1}{(1-q)^2}) \leq \gamma \leq \max(\frac{1}{q^2}, \frac{1}{(1-q)^2})$, the following inequality holds for any matrix $X\in\mathcal{X}$,
\bqs\notag 
\mathbb{E}[\mathcal{L}_{\alpha,q}(X)] - \min_{X\in\mathcal{X}} \mathcal{L}_{\alpha,q}(X) \leq \gamma \left(\mathbb{E}[R_{l^*}(X)] - \min_{X\in\mathcal{X}}\hat{R}_{l^*}(X)\right).
\eqs
\end{lemma}
\begin{proof}
Consider the following two cases: $A_{ij} = 0$ with a probability $1$ if $Y_{ij} = 0$, then
\bqs\notag
& \mathbb{E}[\mathcal{L}_{\alpha,q}(X_{ij})] =  I_{(X_{ij} > q)}, \quad \min_{X\in\mathcal{X}} \mathcal{L}_{\alpha,q}(X_{ij}) = 0; \\
& \mathbb{E}[R_{l^*}(X_{ij})] = X_{ij}^2, \quad \min_{X\in\mathcal{X}}\hat{R}_{l^*}(X_{ij}) = 0.
\eqs

If $X_{ij} \leq q$, it holds with $\gamma X_{ij}^2 \geq 0$; otherwise, $X_{ij}^2 > q^2$ leads to $\gamma X_{ij}^2 > I_{(X_{ij} > q)}$ with $\gamma = 1/q^2$. For the second case if $Y_{ij} = 1$ with the sampling rate $\rho = |\Omega|/S$,
\bqs\notag
 \mathbb{E}[\mathcal{L}_{\alpha,q}(X_{ij})] = \rho\alpha I_{(X_{ij} \leq q)} + (1-\rho) I_{(X_{ij} > q)};
\eqs
\bqs\notag
 \mathbb{E}[R_{l^*}(X_{ij})] = \rho\alpha(X_{ij} - 1)^2 + (1-\rho) X_{ij}^2.
\eqs

If $X_{ij} \leq q$, $(X_{ij} - 1)^2 \geq (q - 1)^2$ leads to $\gamma(X_{ij} - 1)^2 \geq I_{(X_{ij} \leq q)}$ with $\gamma = 1/(1-q)^2$, while $\gamma X_{ij}^2 \geq I_{(X_{ij} > q)} = 0$.

If $X_{ij} > q$, $X_{ij}^2 > q^2$ leads to $\gamma X_{ij}^2 > I_{(X_{ij} > q)}$ with $\gamma = 1/q^2$ while $\gamma (X_{ij} - 1)^2 \geq I_{(X_{ij} \leq q)} = 0$.

Next we compute $\min_{X}\mathcal{L}_{\alpha,q}(X)$ and $\min_{X}\hat{R}_{l^*}(X)$.

If $X_{ij}\leq q$ and $A_{ij}=1$, $\min_{X}\hat{R}_{l^*}(X) = \rho\alpha(1-q)^2$. Thus, $ \min_{X}\hat{R}_{l^*}(X) = \frac{\rho\alpha}{\gamma} I_{(X_{ij}\leq q)}$ where $\gamma = \frac{1}{(1-q)^2}$.

If $X_{ij}>q$ and $A_{ij}=0$, $\min_{X}\hat{R}_{l^*}(X) = (1-\rho)q^2$. Thus $\min_{X}\hat{R}_{l^*}(X) = \frac{1-\rho}{\gamma}I_{(X_{ij} > q)}$ where $\gamma = \frac{1}{q^2}$.

\noindent
Since the errors consist of false positives and false negatives, $\gamma$ holds between $\min\{\frac{1}{q^2}, \frac{1}{(1-q)^2}\}$ and $\max\{\frac{1}{q^2}, \frac{1}{(1-q)^2}\}$.
Combining the above arguments, we conclude the proof.
\end{proof}

Therefore, by further relating $R_{l^*}(X)$ and $\mathcal{L}_{\alpha,q}(X)$ in Lemma \ref{cost_transformation_lamma} with an appropriate value $\alpha$, the following theorem gives us a bound of weighted $sum$.

\begin{thm}\label{sum_theorem}
By setting $\alpha = \frac{\mu_pT_n}{\mu_nT_p}$, $\min(\frac{1}{q^2}, \frac{1}{(1-q)^2}) \leq\gamma\leq \max(\frac{1}{q^2}, \frac{1}{(1-q)^2})$, with probability at least $1 - \delta$, $sum = \mu_p\times recall + \mu_n\times specificity$ is bounded,
\bqs\notag
 \mathbb{E}[sum] & \geq 1 - \min_{X\in\mathcal{X}} \frac{\mu_n}{T_n}\mathcal{L}_{\alpha,q}(X) \\
& - \gamma\alpha \frac{\mu_n}{T_n}\left(\frac{C\epsilon(\sqrt{n} + \sqrt{m} + \sqrt[4]{S})}{mn} + \frac{\sqrt{\log(2/\delta)}}{\sqrt{mn}}\right).
\eqs
\end{thm}

\textbf{Remark:}
The ratio of $T_n/T_p$ is unknown in Theorem \ref{sum_theorem}. To alleviate this issue, we use the weighted \emph{cost} as the evaluation metric and set $\alpha = \frac{C_p}{C_n}$, where $C_p$ and $C_n$ are the predefined cost parameters of false negatives and false positives, respectively. We assume that $C_p + C_n = 1$ and $0 \leq C_n \leq C_p$ since we prefer to improve the accuracy of the rarely observed classes. By this setting, the following theorem shows that the expected cost-sensitive weighted \emph{cost} decays as $O(\alpha/\sqrt{nm})$.

\begin{thm}\label{cost_theorem}
Under the same assumptions in Theorem \ref{sum_theorem}, by setting $\alpha=\frac{C_p}{C_n}$, with probability at least $1 - \delta$, $cost = c_p\times M_p + c_n\times M_n$ is bounded,
\bqs\notag
& \mathbb{E}[cost] - \min_{X\in\mathcal{X}} cost \\
\leq & c_n\gamma\alpha\left(\frac{C\epsilon(\sqrt{n} + \sqrt{m} + \sqrt[4]{S})}{mn} + \frac{\sqrt{\log(2/\delta)}}{\sqrt{mn}} \right),
\eqs
where $C$ is a constant.
\end{thm}

\vspace{-0.002in}
\section{Experimental Results}
In this section we empirically evaluate the performance of our algorithm on three real-world datasets.
We start by introducing our experimental data and benchmark setup, followed by the discussion on the results.

\vspace{-0.002in}
\subsection{Experimental Settings}

\vspace{-0.002in}
\subsubsection{Datasets}

The experiments were conducted on three publicly available datasets:
the MovieLens-100K\footnote{http://grouplens.org/datasets/movielens/}, MovieLens-1M, and EachMovie\footnote{https://grouplens.org/datasets/eachmovie/}.
MovieLens-100K contains 100,000 ratings given by 943 users on 1,682 movies. MovieLens-1M contains 1,000,209 ratings spanning 6,039 users and 3,628 items.
EachMovie contains 2,811,983 ratings entered by 61,265 users on 1,623 movies.
The densities of the U-I matrices derived from the three datasets are $6.3 \times 10^{-2}$, $4.47 \times 10^{-2}$, and $1.91 \times 10^{-2}$ respectively. To simulate the binary feedback scenario on these datasets, we followed the setting of previous studies on implicit feedback~\cite{pan2013gbpr,rendle2009bpr} and treated the ratings larger than 3 as observed positive feedback.

\vspace{-0.002in}
\subsubsection{Evaluation Metrics}

Each dataset is randomly partitioned into two non-overlapping sets for training and testing. Given the observation matrix $A=[\a_1,\ldots,\a_m]$, for each user $\a_i$, $80\%$ of its historical items $\a_{i}^+$ are randomly chosen as the training data, and the remaining $20\%$ of $\a_{i}^+$ are used for testing. In order to fairly compare these algorithms, we randomly selected the training samples for each dataset, repeated the random partition for 5 times, and reported averaged results.

In the implicit feedback scenario, we followed the setting of~\cite{pan2013gbpr,liu2015boosting} by measuring the performance of different methods with four evaluation metrics: NDCG@N, P@N, R@N and F1-score@N, which are widely used in recommender systems~\cite{su2009survey}.
1) NDCG@N is the normalized discounted cumulative gain in the ranking. It reflects the usefulness or gain in the ranking list.
2) R@N is the recall metric in the top-N ranked list.
3) P@N is the precision metric in the top-N ranked list.
4) F1-score@N is the weighted harmonic mean of precision and recall.
In recommender systems, users are usually interested in a few top-N ranked items. In our work, we set N$= \{5,10,15\}$.
Specifically, the higher these measures, the better the performance of an algorithm is.
For each metric, we first computed the performance for each user on the testing data, then reported the averaged accuracy over all the users.

\subsubsection{Baselines}

We compared the proposed algorithms with four strong baseline algorithms: PopRank, WRMF, BPRMF, and MC-Shift.
1) PopRank is a naive baseline that recommends items to users purely based on the popularity of the items.
2) WRMF is a very strong matrix factorization model for item prediction~\cite{hu2008collaborative}. This method outperforms neighborhood-based models for item prediction on implicit feedback datasets.
3) BPRMF is a quite strong matrix factorization ranking model for item recommendation with implicit feedback~\cite{rendle2009bpr}. It is a pairwise learning-to-rank approach.
4) MC-Shift is the latest positive-unlabeled learning algorithm for matrix completion~\cite{hsieh2015pu}.

We adopted cross-validation to choose the parameters for all the evaluated algorithms.
In the MF-based methods (i.e. WRMF, BPRMF, CSRR-e), the number of the latent dimension was tuned from $\{10, 15, \ldots,50\}$.
For WRMF, we tuned $\lambda$ by grid search from $\{2^{-5},\ldots,2^0\}$ and set $C = 1$.
For BPRMF, the regularization parameters were tuned from $\{10^{-6},\ldots,10^0\}$ and the learning rate from $\{2^{-6},\ldots,2^2\}$.
For MC-Shift, we followed its setting~\cite{hsieh2015pu}, and tuned $\lambda$ and $\rho$ by a random validation set.
For CSRR, $\alpha = \frac{C_p}{C_n}$ was tuned with $C_p$ from $\{0.5,0.55,\ldots,0.95\}$ with a stepsize of $0.05$, while $\eta$ and $\lambda$ were tuned from $\{10^{-5},\ldots, 10^2\}$ on a heldout random sampling set.


\vspace{-0.01in}
\subsection{Comparison Results}

The performance of CSRR and other baselines is summarized in Table \ref{comparison_result}. We only adopt the CSRR-e on EachMovie due to a high runtime of CSRR-I/II. We make the following observations:
1) It can be seen that CSRR outperforms or at least compares favorably with all other baselines over all the datasets. In particular, the improvement over the MF-based methods (BPRMF and WRMF) are significant, e.g. at least $10\%$ improvement on NDCG@N of MovieLens-100K and $5\%$ on F-score@N of MovieLens-1M, for $ N\in\{5,10,15\}$.
2) Compared with cost-insensitive techniques, CSRR attains significant improvement on all three datasets. For instance, CSRR-I outperforms PopRank, BPRMF and MC-Shift by $84.5\%$, $17.2\%$ and $4.03\%$ on NDCG@5 of MovieLens-100K. These results show that the accuracy of item recommendation can be largely improved by jointly estimating outliers and cost-sensitive classification.

\begin{table*}[t]
\centering
\small
\linespread{0.8}
\caption{Four evaluation metrics over the MovieLens-100K, MovieLens-1M and EachMovie Datasets}
\label{comparison_result}
\begin{tabular}[2.1\textwidth] {|c|c|c|c|c|c|c|c|c|c|c|}
\hline
\multirow{2}{*}{Algorithm} & \multicolumn{10}{|c|}{\emph{MovieLens-100K}}  \\
\cline{2-11}
& R@5	& R@10   & P@5 & P@10 & F-score@5 & F-score@10 & F-score@15 &  NDCG@5 & NDCG@10 & NDCG@15 \\
\hline\hline
PopRank	& 0.0634 & 0.1192 & 0.1661 & 0.1569	& 0.0918  & 0.1355 & 0.1524 & 0.3935 & 0.4387  &  0.4507\\ \hline
BPRMF	& 0.1466 & 0.2325 & 0.3597 & 0.2965 & 0.2083  & 0.2606 & 0.2793 & 0.6297 & 0.6463  &  0.6463 \\ \hline
WRMF    & \textbf{0.1546} & \textbf{0.2411} & 0.3640  & 0.3008 & 0.2170 & 0.2677  & 0.2776 & 0.6525  & 0.6607  &  0.6598 \\ \hline
MC-Shift& 0.1301 & 0.2111 & 0.4522 & 0.3866	& 0.2021  & 0.2731 & 0.3035 & 0.7097 & 0.7146  &  0.7133 \\ \hline
CSRR-e& 0.1311 & 0.2026 & 0.4487 & 0.3700	& 0.2030  & 0.2618 & 0.2908 & 0.7093 & 0.7061  &  0.7317 \\ \hline
CSRR-II& 0.1401 & 0.2171 & 0.4716  &0.3927 &0.2160 & 0.2796 & 0.3100 & 0.7372  & 0.7380 & 0.7308 \\ \hline
CSRR-I & 0.1409 & 0.2173 & \textbf{0.4736} & \textbf{0.3943} & \textbf{0.2172} &  \textbf{0.2802} & \textbf{0.3114} & \textbf{0.7382}  & \textbf{0.7390} &  \textbf{0.7323}  \\ \hline
\hline
\multirow{2}{*}{Algorithm} & \multicolumn{10}{|c|}{\emph{MovieLens-1M}}  \\
\cline{2-11}
& R@5	& R@10   & P@5 & P@10 & F-score@5 & F-score@10 & F-score@15 & NDCG@5 & NDCG@10 & NDCG@15  \\
\hline\hline
PopRank		& 0.0421	& 0.0713		& 0.1879 	& 0.1647	&  0.0688 &  0.0995 & 0.1189 & 0.3763 & 0.4084  &  0.4241\\ \hline
BPRMF		& 0.0949	& 0.1575		& 0.3590 	& 0.3106	&  0.1501 &  0.2090 & 0.2364 & 0.6044 & 0.6221  &  0.6247\\ \hline
WRMF        & 0.1083	& 0.1759		& 0.3770 	& 0.3236	&  0.1683 &  0.2279 & 0.2535 & 0.6370 & 0.6532  &  0.6523\\ \hline
MC-Shift    & 0.1119	& 0.1774		& 0.3941 	& 0.3341	&  0.1743 &  0.2317 & 0.2564 & 0.6525 & 0.6650  &  0.6641\\ \hline
CSRR-e      & 0.1139	& 0.1808		& 0.3978 	& 0.3381	&  0.1771 &  0.2356 & 0.2661 & 0.6579 & 0.6697  &  0.6680\\ \hline
CSRR-II     & 0.1161    & \textbf{0.1861}  &0.3954	    & 0.3378    & 0.1795  &  0.2400 & 0.2663 & 0.6553 & 0.6683  &  0.6681\\ \hline
CSRR-I     & \textbf{0.1165}   & 0.1850  & \textbf{0.4007}	& \textbf{0.3400} & \textbf{0.1805} & \textbf{0.2400} & \textbf{0.2700} & \textbf{0.6591} & \textbf{0.6711} & \textbf{0.6716}\\ \hline
\hline
\multirow{2}{*}{Algorithm} & \multicolumn{10}{|c|}{\emph{EachMovie}}  \\
\cline{2-11}
& R@5	& R@10  & P@5 & P@10 & F-score@5 & F-score@10 & F-score@15  & NDCG@5 & NDCG@10 & NDCG@15 \\
\hline\hline
PopRank	& 0.178	         & 0.265	      & 0.1751 	        & 0.1433           & 0.1765 & 0.1860 & 0.1861 & 0.3949   & 0.4241 & 0.4371 \\ \hline
BPRMF	& 0.3251         & \textbf{0.4734}& 0.3028          & 0.2462           & 0.3135 & \textbf{0.3239} & \textbf{0.3067} & 0.6056	 & 0.6245 &  0.6311  \\ \hline
WRMF    & 0.3331         & 0.4703         & 0.3025          & 0.2424           & 0.3175 & 0.3199 & 0.2998 & 0.6085   & 0.6248  & \textbf{0.6316} \\ \hline
MC-Shift& 0.3213	     & 0.4562	      & 0.3056 	        & 0.2362	       & 0.3133    & 0.3112  & 0.2732 & 0.5985   & 0.6081  &  0.6086\\ \hline
CSRR-e & \textbf{0.3334} & 0.4620        & \textbf{0.3177} & \textbf{0.2462}  & \textbf{0.3254} & 0.3212 & 0.2949 & \textbf{0.6108}     & \textbf{0.6248}  &  0.6218\\ \hline
\end{tabular}
\end{table*}

\begin{figure*}
\centering
\caption{Evaluation of the proposed algorithms under varying weights of cost bias $\alpha$}
\label{Sensitivity-alpha}
\subfigure {\includegraphics[width=0.24\textwidth,height=3.65cm]{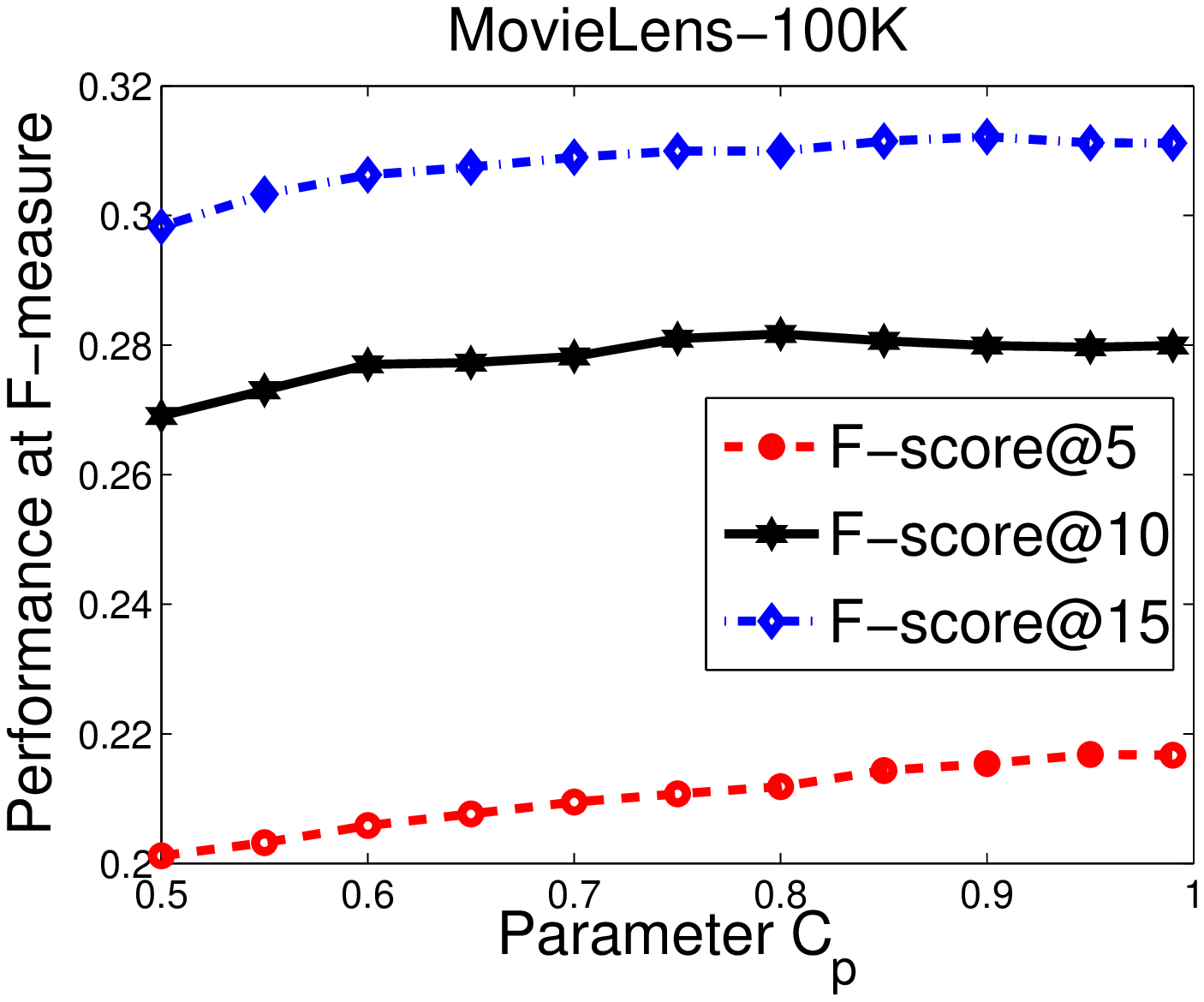}}
\subfigure {\includegraphics[width=0.24\textwidth,height=3.65cm]{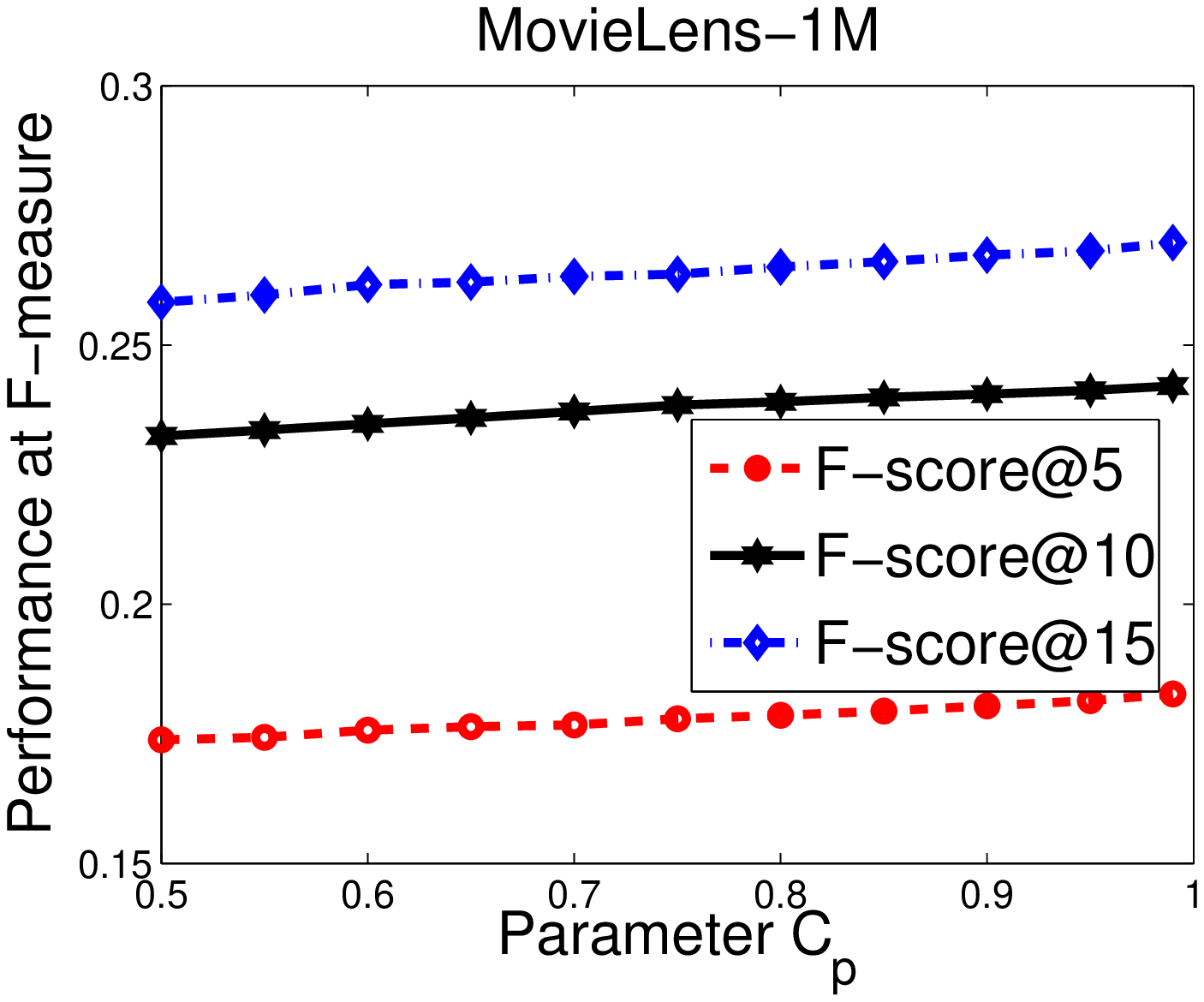}}
\subfigure {\includegraphics[width=0.24\textwidth,height=3.65cm]{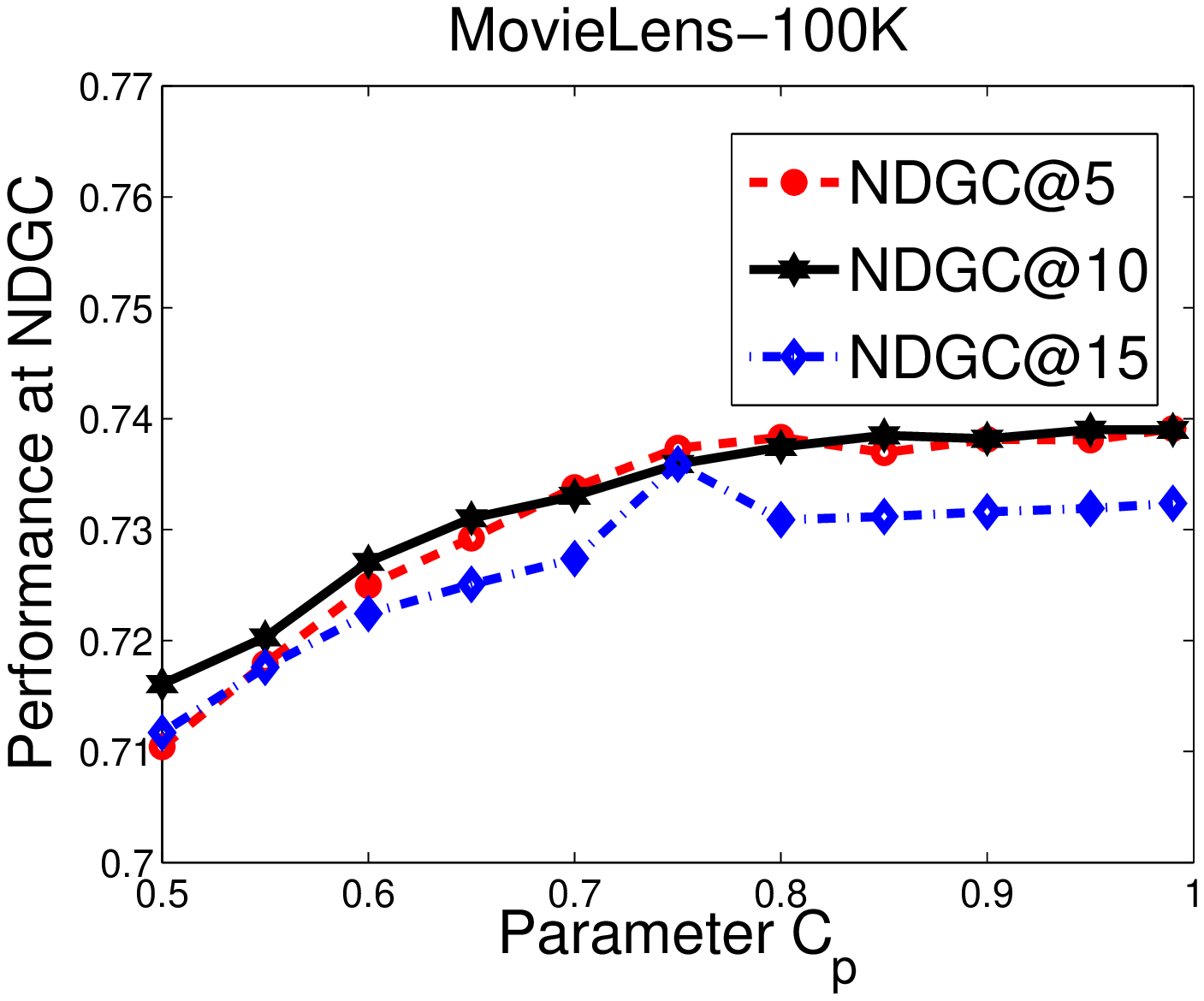}}
\subfigure {\includegraphics[width=0.24\textwidth,height=3.65cm]{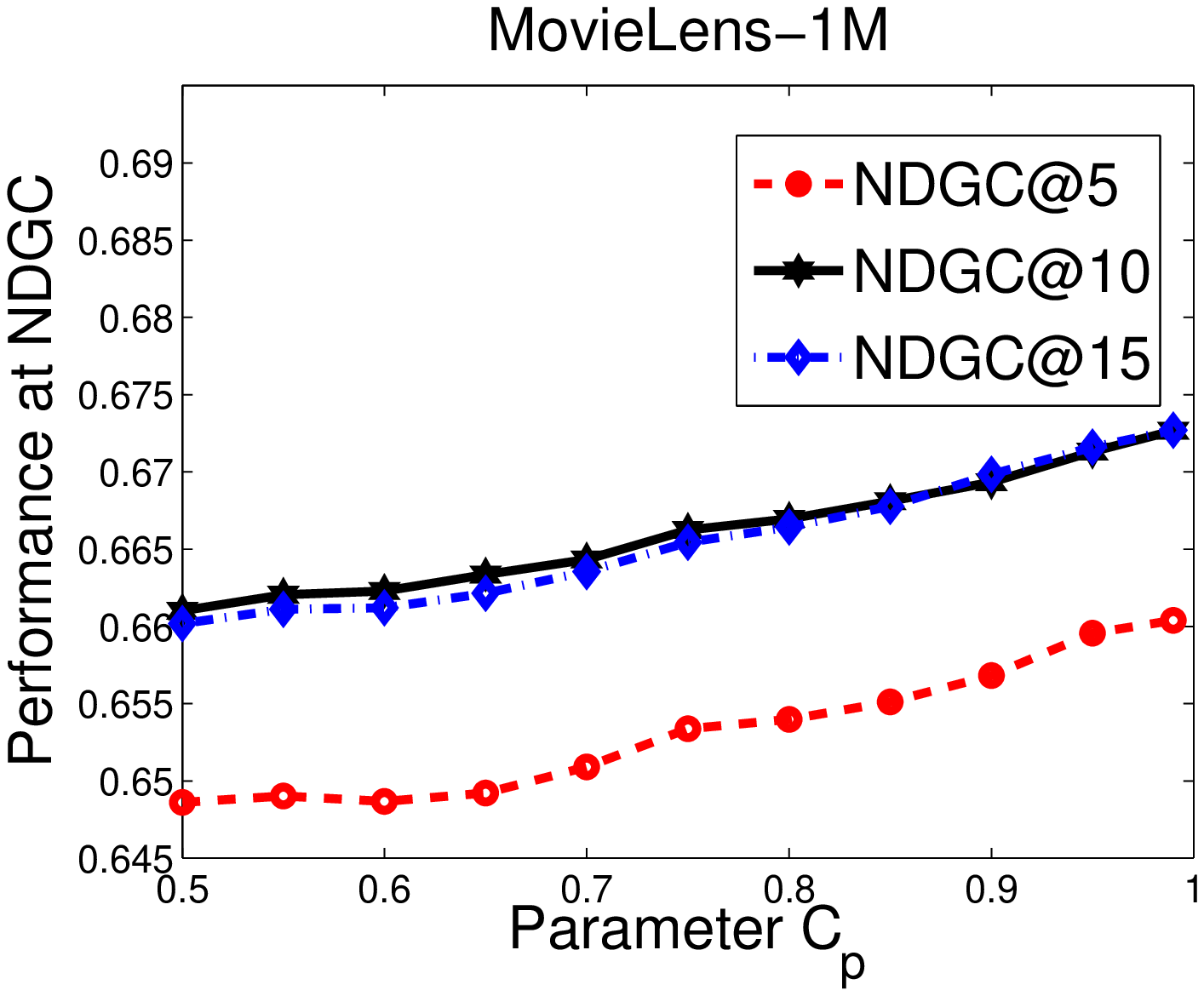}}
\caption{Evaluation of the proposed algorithms under varying weights of the learning rate $\eta$}
\label{Sensitivity-eta}
\subfigure {\includegraphics[width=0.24\textwidth,height=3.65cm]{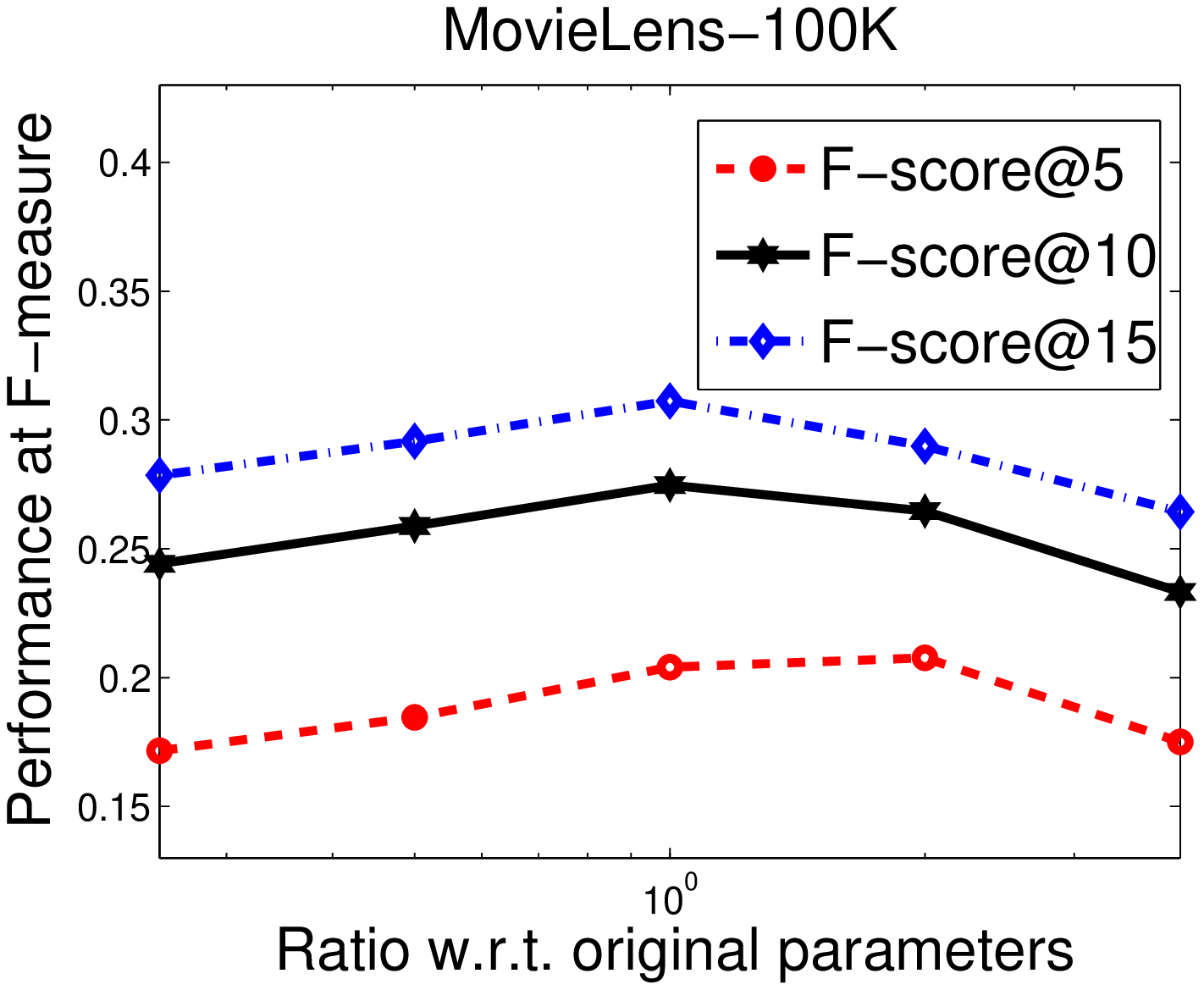}}
\subfigure {\includegraphics[width=0.24\textwidth,height=3.65cm]{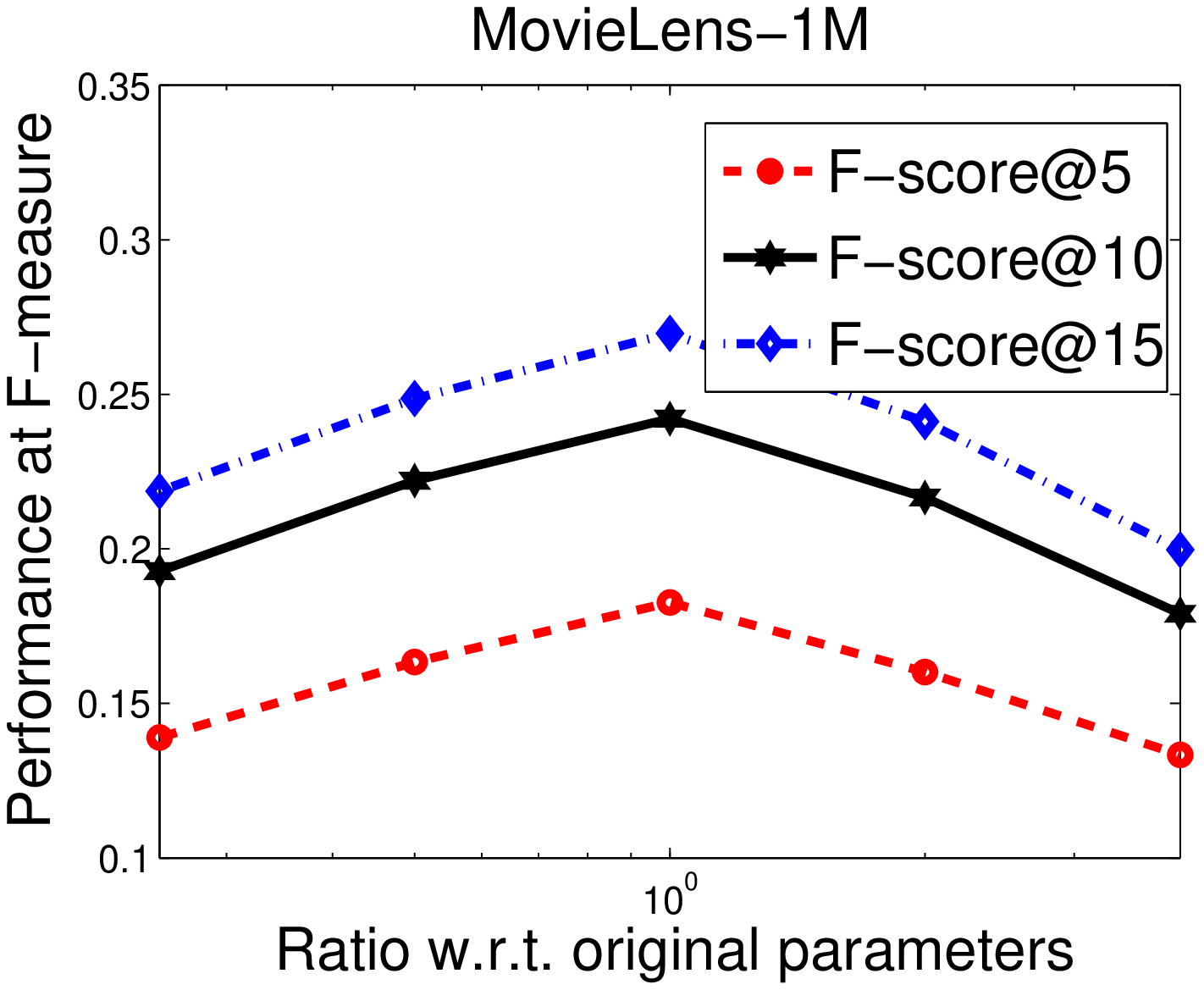}}
\subfigure {\includegraphics[width=0.24\textwidth,height=3.65cm]{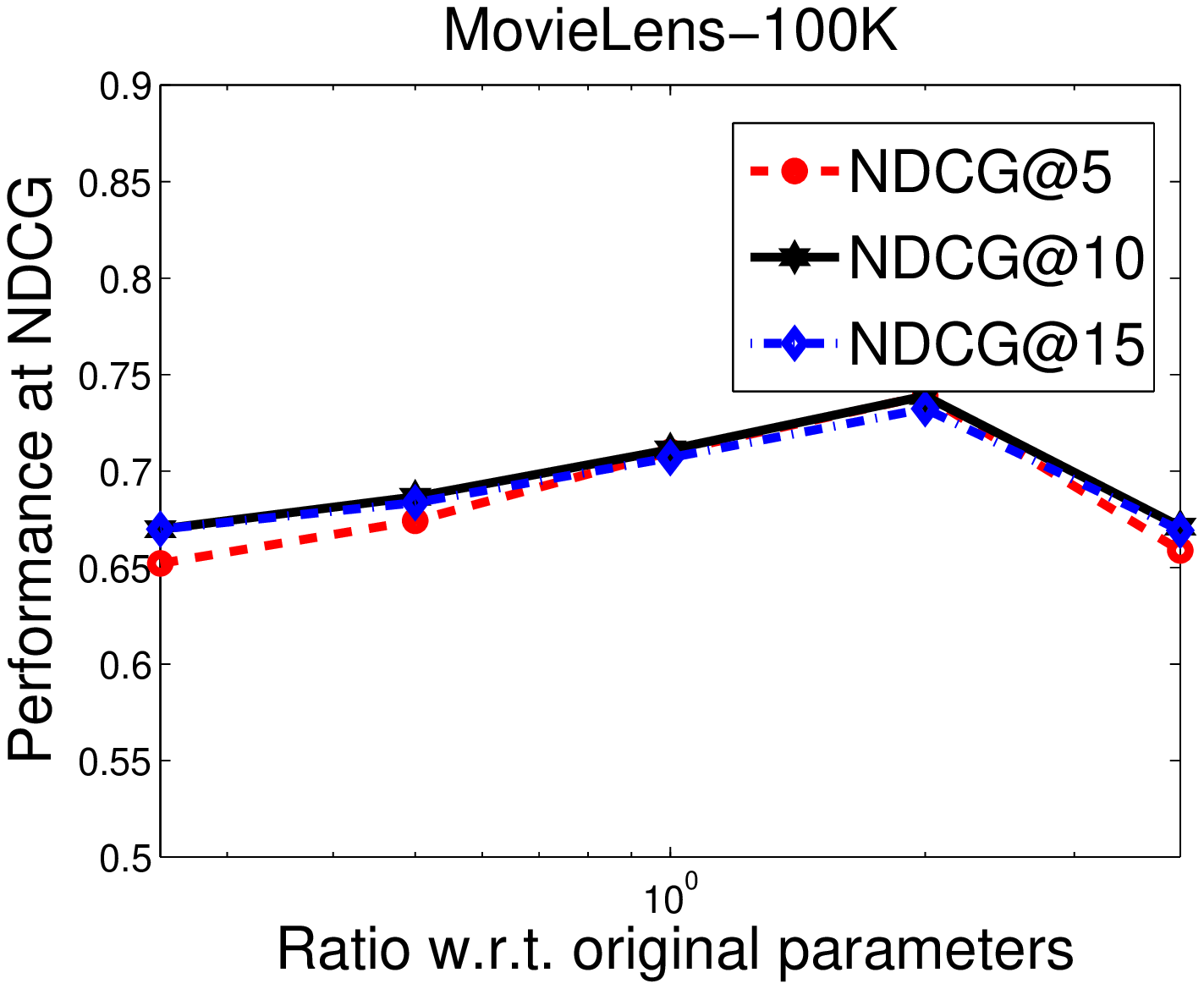}}
\subfigure {\includegraphics[width=0.24\textwidth,height=3.65cm]{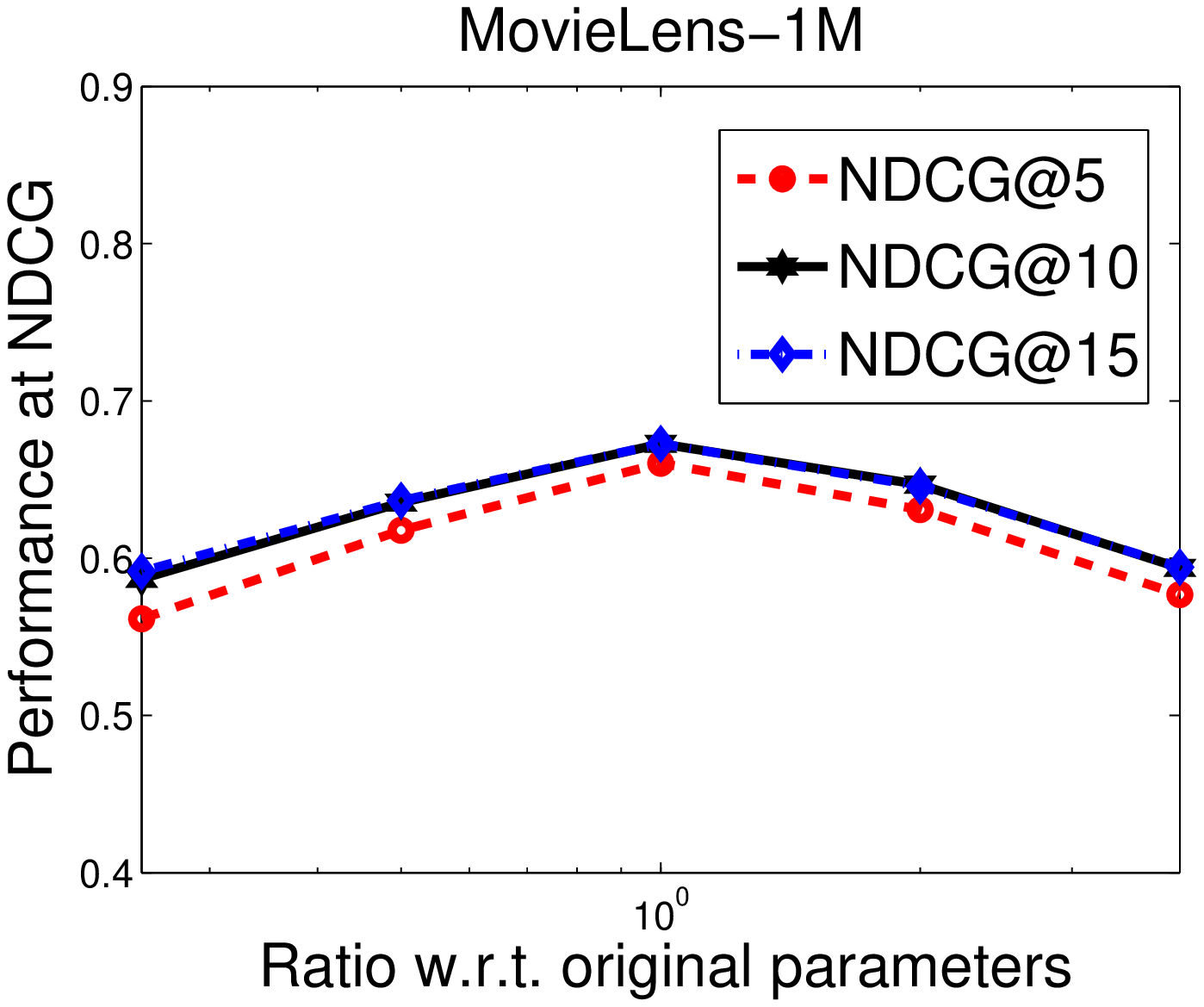}}
\caption{Evaluation of the proposed algorithms under varying weights of $\lambda_1$ and $\lambda_2$}
\label{Sensitivity-figure}
\subfigure {\includegraphics[width=0.24\textwidth,height=3.65cm]{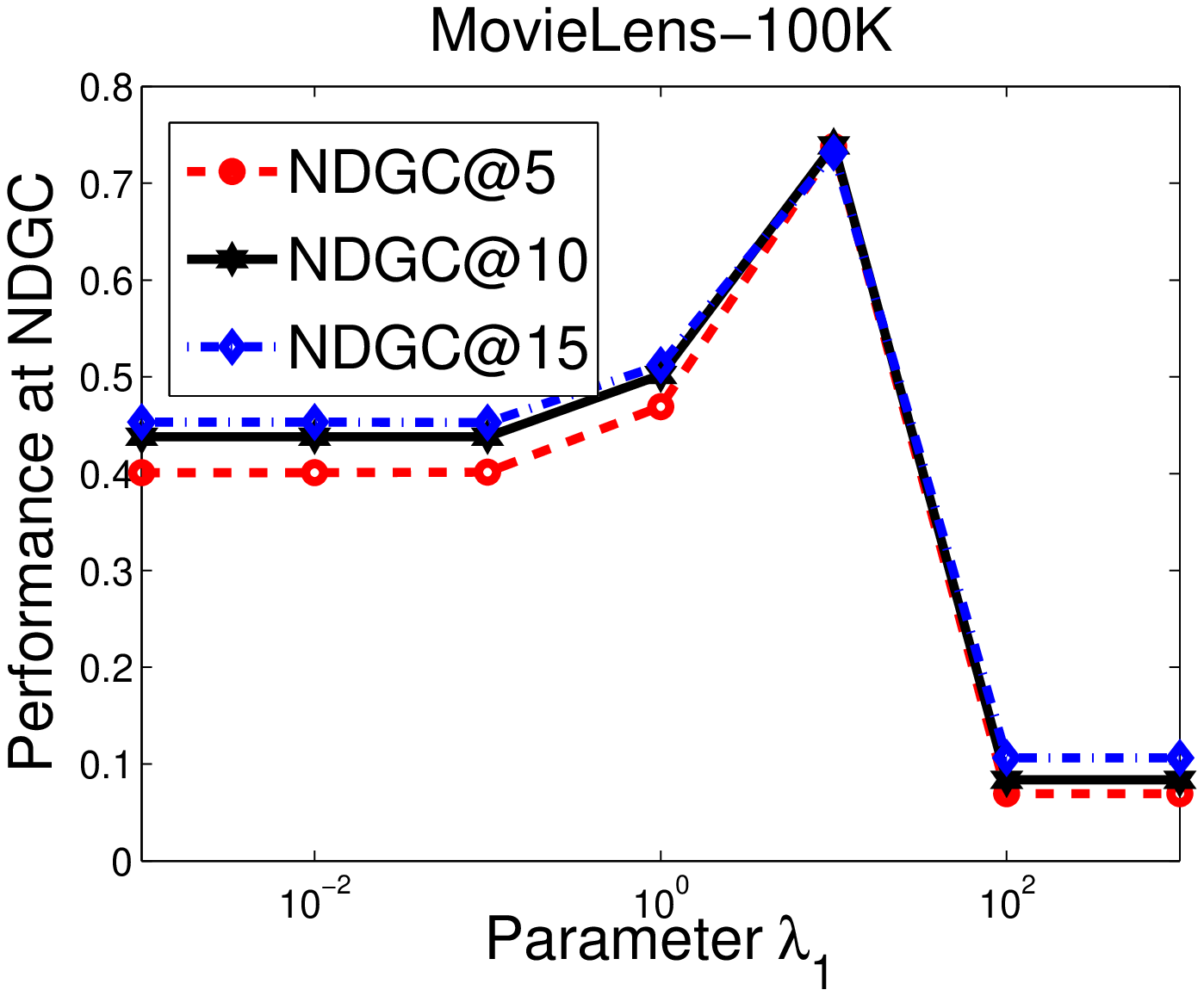}}
\subfigure {\includegraphics[width=0.24\textwidth,height=3.65cm]{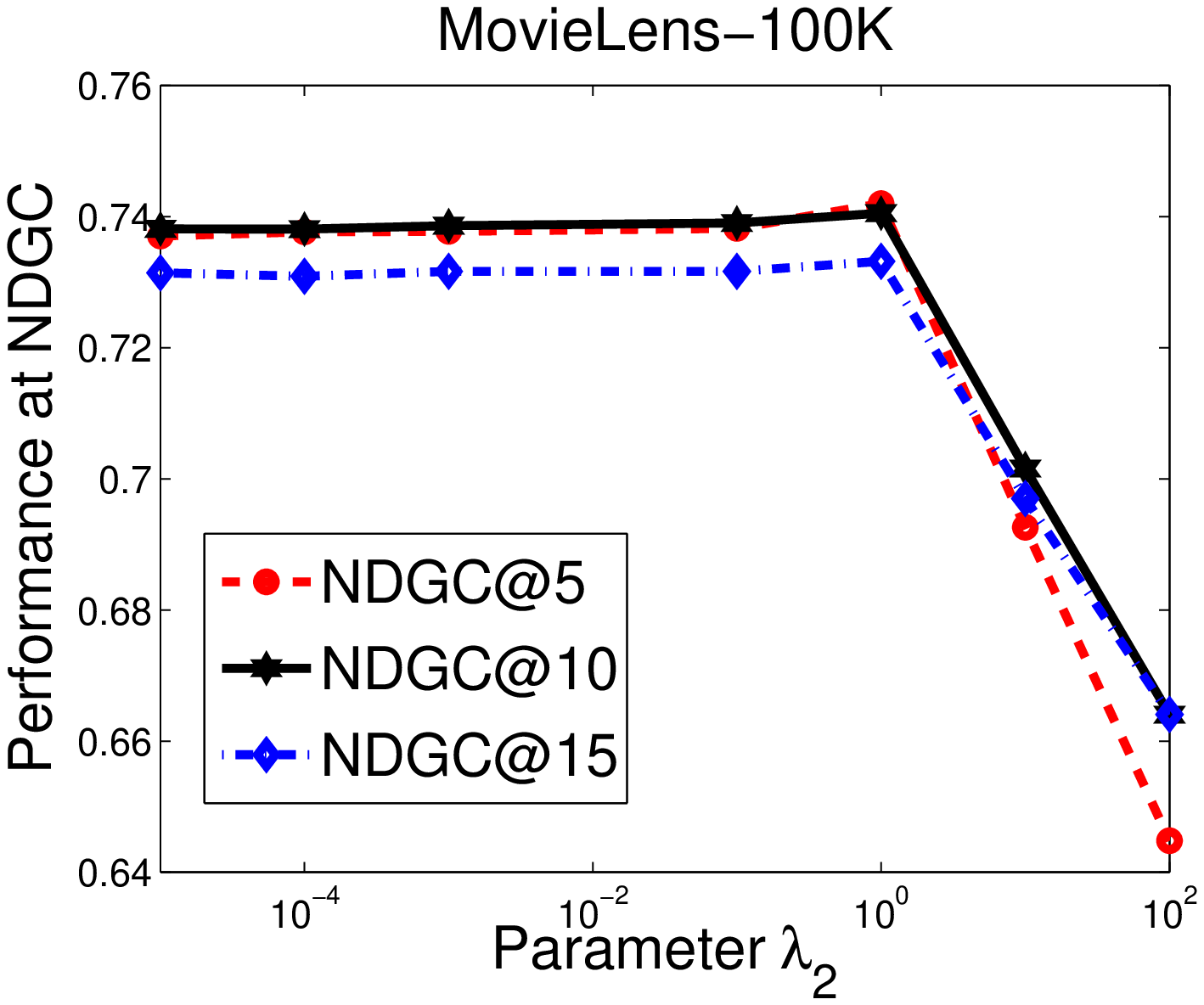}}
\subfigure {\includegraphics[width=0.24\textwidth,height=3.65cm]{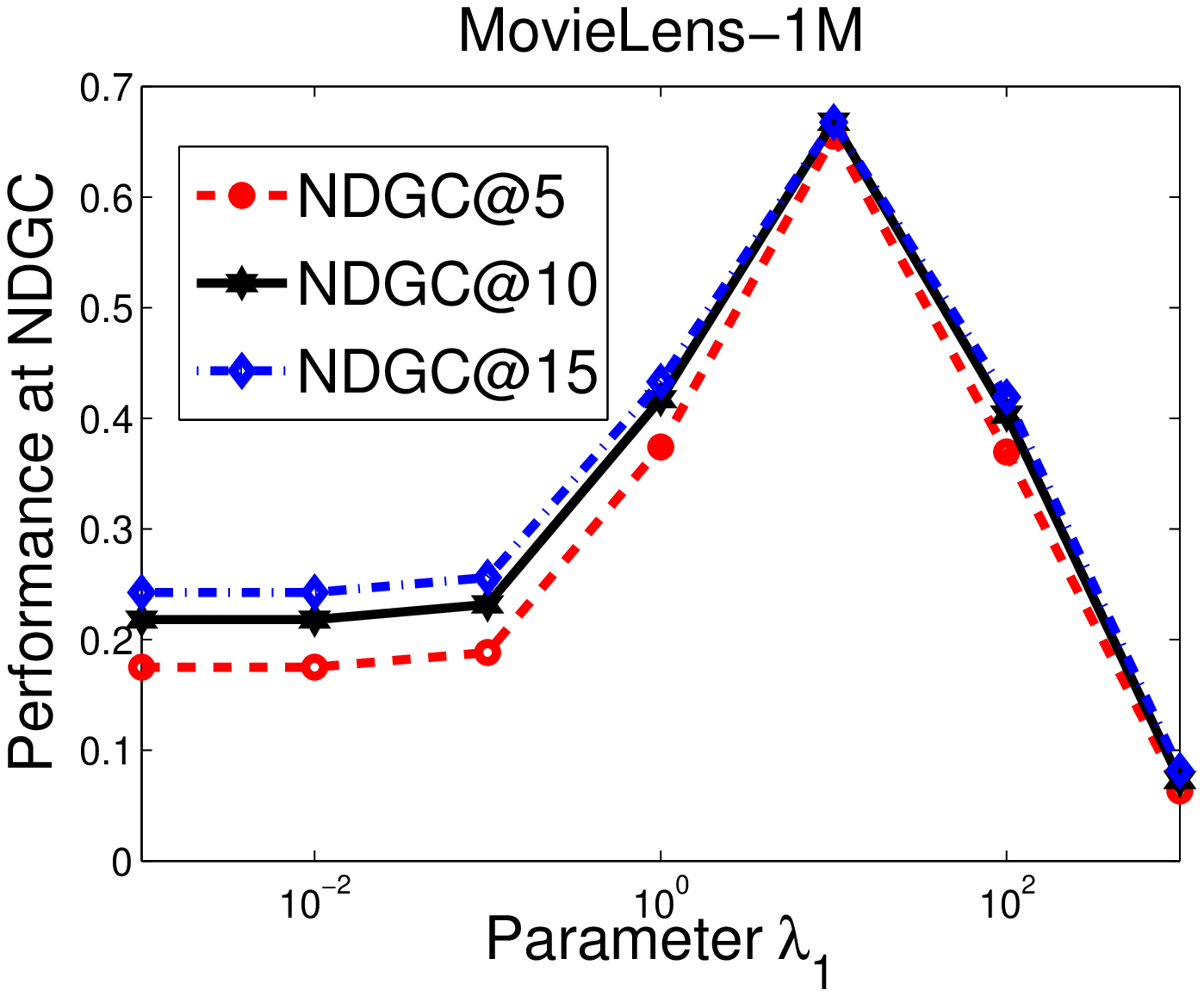}}
\subfigure {\includegraphics[width=0.24\textwidth,height=3.65cm]{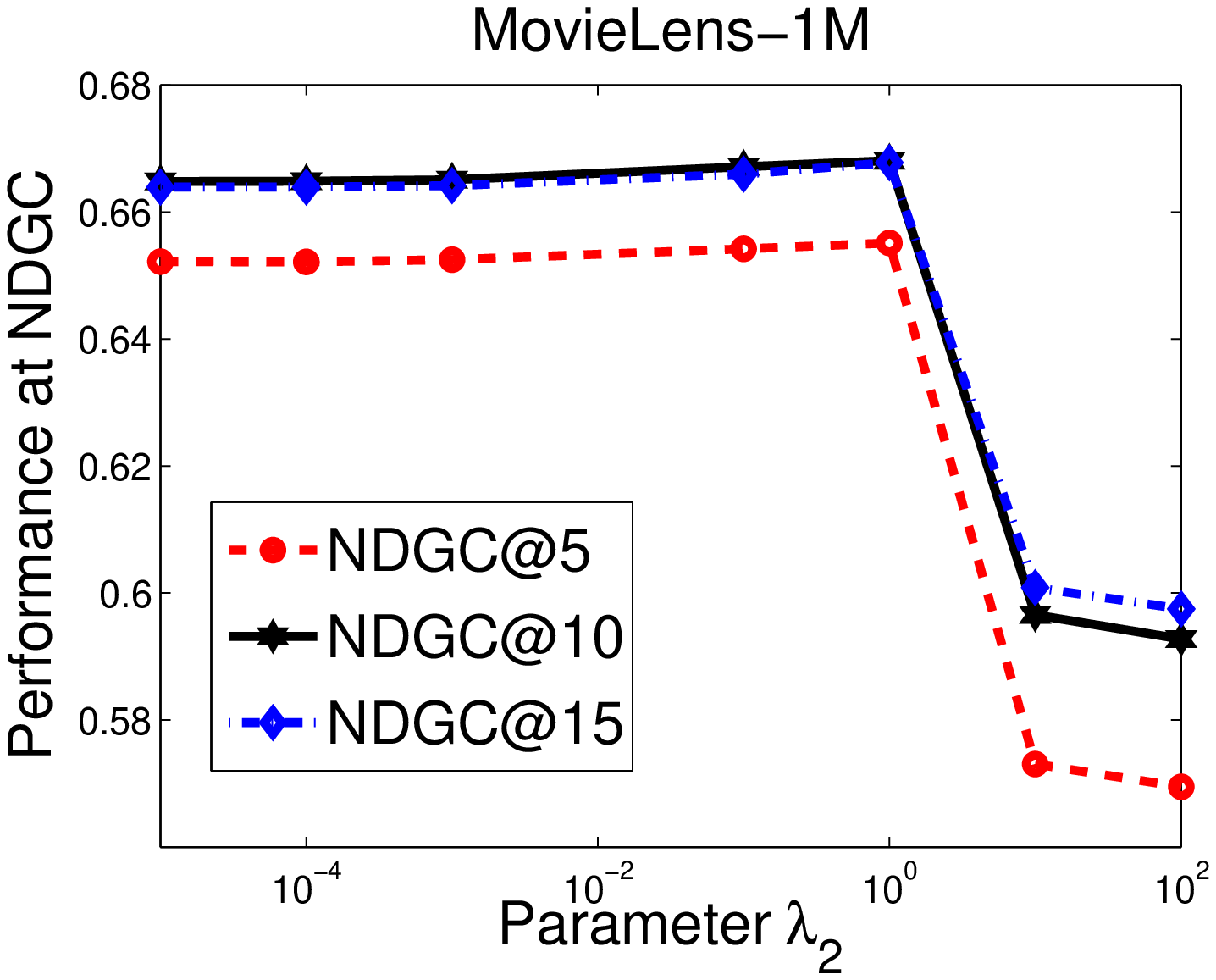}}
\end{figure*}

\vspace{-0.01in}
\subsection{Parameter Sensitivity Analysis}

We next conduct parameter sensitive study to evaluate our algorithms. We use the CSRR-I as a case study as similar observations are obtained on CSRR-II/e.

\subsubsection{Analysis for Cost-Sensitive Bias}

To demonstrate the efficiency of cost-sensitive learning, we aim to evaluate the algorithm under varying cost-sensitive bias $\alpha$. We set the parameter $\alpha$ by tuning $c_p$ within $\{0.5,\ldots,0.95\}$.
Fig. \ref{Sensitivity-alpha} shows the evaluation results.
With a small value of $c_p$, it is observed that the performance is degraded over all metrics by such cost-insensitive classification.
The possible reason is that the positive predictions are difficult, i.e. their latent features are similar to positive examples but they are actually negatives. There can be many such cases as the dataset is imbalanced. Updating with those samples will push the decision boundary to the side that includes fewer positive samples and consequentially harm performance. Thus, using cost-sensitive learning with a large value of $c_p$ will reduce the impact of such samples and lead to more ``aggressive" ($\ell^I$) or ``frequent" ($\ell^{II}$) update of the positive class, which is consistent with observations both in theory and in practice.

%

\subsubsection{Analysis for Learning Rate}

We also examined the parameter sensitivity of the learning rate parameter $\eta$. In particular, we set the learning rate as
a factor in $\{2^{-2}, 2^{-1},\ldots,2^2\}$ times the original learning rate used in the comparison results above, and reported the performance under the varied learning rate settings. Fig. \ref{Sensitivity-eta} shows the evaluation results. We observed that our algorithms perform quite well on a relatively large parameter space of the learning rate.

\subsubsection{Analysis for Robust Matrix Decomposition}

We studied the impact of the regularization parameter pair $(\lambda_1,\lambda_2)$ in our method. The algorithm with a large $\lambda_1/\lambda_2$ would prefer the common/personalized structure.
In particular, by fixing $\lambda_2 = 0.1$, and varying $\lambda_1$ in the range $[10^{-3},\ldots,10^3]$, or by fixing $\lambda_1 = 0.1$ and varying $\lambda_2$ in the range $[10^{-5},\ldots,100]$, we studied how the parameter $\lambda_1$ or $\lambda_2$ affects the performance.
The comparison results is present in Fig. \ref{Sensitivity-figure}. We observed that the proposed method does not perform well with either a large or a small value of $\lambda_1$. It indicates that increasing the impact of low-rank structure may hurt the observed matrix structure. Thus, better performance can be achieved by balancing the impact of the outliers and the low-rank latent structure.

In addition, the accuracy can be improved when $\lambda_2$ is increased, inferring that the learned $V$ is effective to reduce the gross outliers. Moreover, when we disable the outlier estimation, i.e. $V = 0$, the results in Table \ref{Anaysis_result} show that the performance without $V$ (CSRR-I/$V_0$) is obviously degraded, which infers that without $V$ to reduce the negative impact of outliers, a low-rank basis could be corrupted and prevented from reliable solution.

\section{Conclusion}

In this paper we proposed a robust cost-sensitive learning for recommendation with implicit feedback. It can estimate outliers in order to recover a reliable low-rank structure. To minimize the asymmetric cost of errors from imbalanced classes, a cost-sensitive learning is proposed to impose the different penalties in the objective for the observed and unobserved entries. To our knowledge, this is the first algorithm that combines cost-sensitive classification and an additive decomposition for recommendation. The derived non-smooth optimization problem can be efficiently solved by an accelerated projected gradient schema with a closed-form solution. We proved that CSRR can achieve a lower error rate when $\sqrt{mn}$ is large enough. Finally, the promising empirical results on the real-world applications of recommender systems validate the effectiveness of the proposed algorithm compared to other state-of-the-art baselines. In the future work, such robust cost-sensitive learning could be applied to online graph classification~\cite{yang2016efficient,yang2015min,yang2015aggressive} and disease-gene identification problem~\cite{yang2011inferring}.

\begin{table}[t]
\centering
\small
\linespread{0.8}
\caption{Analysis for Robust Matrix Decomposition}
\label{Anaysis_result}
\begin{tabular}[2.1\textwidth] {|c|c|c|c|c|}
\hline
\multirow{2}{*}{Algorithm} & \multicolumn{4}{|c|}{\emph{MovieLens-100K}}  \\
\cline{2-5}
& F-score@5 & F-score@15 &  NDCG@5 & NDCG@15 \\
\hline\hline
CSRR-I/$V_0$ &0.2131 & 0.3010 & 0.7304  & 0.7193 \\ \hline
CSRR-I &0.2172 & 0.3114 & 0.7382  & 0.7323  \\ \hline
\hline
\multirow{2}{*}{Algorithm} & \multicolumn{4}{|c|}{\emph{MovieLens-1M}}  \\
\cline{2-5}
& F-score@5  & F-score@15 & NDCG@5 & NDCG@15  \\
\hline\hline
CSRR-I/$V_0$ & 0.1726  & 0.2607 & 0.6363  &  0.6524\\ \hline
CSRR-I & 0.1805  & 0.2700 & 0.6591  & 0.6716\\ \hline
\end{tabular}
\end{table}

\section{Appendix}

\textbf{Proof of Lemma~\ref{optimation_UV}}

\begin{proof}
Given that $W = \begin{bmatrix} U \\ V \end{bmatrix}$, we obtain that
\bqs\notag 
& \frac{1}{2\eta}\sum_{i=1}^m \mathcal{D}(\w_i,\w^t_i) = \frac{1}{2\eta}\sum_{i=1}^m \|\w_i - \w^t_i\|_F^2 = \frac{1}{2\eta}\|W - W^t\|^2_F \\
   & \overset{(\ref{decompositionW})} {=} \frac{1}{2\eta}\left\| \begin{bmatrix} U - U^t \\ V - V^t \end{bmatrix} \right\|_F^2 = \frac{1}{2\eta}\sum_{i=1}^m(\|\u_i - \u_i^t\|^2_F + \|\v_i - \v_i^t\|^2_F).
\eqs
The linearized gradient form can be rewritten,
\bqs\notag 
 & \sum_{i=1}^m\langle \nabla \ell^*_t,\w_i - \w_i^t \rangle
  \overset{(\ref{decompositionW})} {=}  \sum_{i=1}^m \left\langle \begin{bmatrix} \partial\ell^*_t / \partial \u_i^t \\ \partial\ell^*_t / \partial \v_i^t \end{bmatrix} , \begin{bmatrix} \u_i - \u_i^t \\ \v_i - \v_i^t \end{bmatrix} \right\rangle \\
 = & \sum_{i=1}^m \left\langle \frac{\partial\ell^*_t}{\partial \u_i^t}, \u_i - \u_i^t \right\rangle + \sum_{i=1}^m\left\langle \frac{\partial\ell^*_t}{\partial \v_i^t}, \v_i - \v_i^t \right\rangle.
\eqs
Substituting above two equations into problem (\ref{linearization_projection}) with some manipulations, we have
\bqs\notag
 (U^{t+1},V^{t+1})
\triangleq & \underset{U,V}{\operatorname{argmin}} \frac{1}{2\eta} \left( \|U - \hat{U}^t\|^2_F  + \|V - \hat{V}^t\|^2_F \right) \\
 & + \lambda_1\|U\|_{*} + \lambda_2\|V\|_{1},
\eqs
where $\hat{U}^t = U^t - \eta\frac{\partial \ell^*_t}{\partial U^t}$, and $\hat{V}^t = V^t - \eta\frac{\partial \ell^*_t}{\partial V^t}$. Due to the decomposability of objective function above, the solution for $U$ and $V$ can be optimized separately, thus we can conclude the proof.
\end{proof}

\textbf{Proof of Theory \ref{expectlossbound_thm}}
\begin{proof}
Given that $X_{ij}\in[0,1]$ and $\alpha>1$, each $\ell^{I}(X_{ij},A_{ij})$ can be either $X_{ij}^2$ or $\alpha(X_{ij}-1)^2$, when changing one random variable $A_{ij}$, in the worst case the quantity $\ell^{I}(X_{ij},A_{ij})$ can be changed by
\bqs\notag
& |X_{ij}^2 - \alpha(X_{ij}-1)^2|  = |(\alpha-1)(X_{ij} - \frac{\alpha}{\alpha-1})^2 - \frac{\alpha}{\alpha-1}| \\
& \leq \max\left(\ell^{I}(1,A_{ij}),\ell^{I}(0,A_{ij})\right) = \alpha.
\eqs
$\ell^{II}(X_{ij},A_{ij})$ has a similar bound as above.
Motivated by Theorem 1 in~\cite{rosenberg2007rademacher}, using McDiarmid's Theorem~\cite{alpaydin2014introduction} with at least probability $1 - \delta/2$,
\bqs\notag
\mathbb{E}[R_{\ell^*}(X)] - R_{\ell^*}(X) \leq  \tilde{R}_{\ell^*}(X) + \frac{\alpha\sqrt{\log(2/\delta)}}{\sqrt{mn}},
\eqs
where $\tilde{R}_{\ell^*}(X) = \frac{2}{mn}\mathbb{E}_{\sigma}\left[\sup_X\left(\sum_{ij}\sigma_{ij}\ell^*(X_{ij},A_{ij})\right)\right]$
and $Pr(\sigma_{ij}=+1) = Pr(\sigma_{ij}=-1) = 0.5$. \\
Since $X_{ij}\in[0,1]$, and the Lipschitz constant for $\ell^*(X_{ij},A_{ij})$ is at most $\alpha$, i.e. $|\ell^*(X)|\leq\alpha|X| + |\ell^*(0)|$,
\bqs\notag
\tilde{R}_{\ell^*}(X)  \leq & \frac{2\alpha}{mn}\mathbb{E}_{\sigma}\left[\sup_X\sum_{ij}\sigma_{ij}X_{ij} + \sup_X\sum_{ij,A_{ij}=1}\sigma_{ij} \right] \\
    \leq & \frac{2\alpha}{mn}\mathbb{E}_{\sigma}\left[\sup_X\|\sigma\|_2\|X\|_*\right],
\eqs
where $\mathbb{E}_{\sigma}[\sum_{ij}\sigma_{ij}]=0$, similar in the proof of Theorem 3 in~\cite{hsieh2015pu}. Applying main Theorem in~\cite{latala2005some} to $\mathbb{E}[|\sigma\|_2]$ where $\sigma_{ij}$ is independent mean zero entry, we have
\bqs\notag
& \mathbb{E}[|\sigma\|_2] \\
\leq & C\left(\max_i\sqrt{\sum_{j}\mathbb{E}(X_{ij}^2)}+\max_j\sqrt{\sum_{i}\mathbb{E}(X_{ij}^2)}
 + \sqrt[4]{\sum_{ij}\mathbb{E}(X^4_{ij})}\right)\\
\leq & C\left(\sqrt{m} + \sqrt{n} + \sqrt[4]{S}\right),
\eqs
where $C$ is some universal constant and we conclude our proof.
\end{proof}

{
\bibliographystyle{IEEEtranS}
\bibliography{aaai2016}
}

\end{document}